\documentclass[times,final]{elsarticle}

\usepackage{jcomp}
\usepackage{framed,multirow}

\usepackage{amssymb}
\usepackage{latexsym}
\usepackage{amsmath}
\usepackage{geometry}
\usepackage{ulem}
\usepackage{graphicx}
\usepackage{algpseudocode}
\usepackage{algorithmicx}
\usepackage{algorithm}
\usepackage{caption}
\usepackage{subcaption}
\usepackage{amsfonts}
\usepackage{float}
\usepackage{url}
\usepackage{xcolor}
\newcommand{\norm}[1]{\left\lVert#1\right\rVert}

\newtheorem{thm}{Theorem}[section]
\newtheorem{proof}{Proof}[section]
\definecolor{newcolor}{rgb}{.8,.349,.1}

\journal{Journal of Computational Physics}

\begin{document}

\verso{Jiahao Zhang \textit{etal}}

\begin{frontmatter}

\title{Energy-Dissipative Evolutionary Deep Operator Neural Networks}

\author[1]{Jiahao \snm{Zhang}\fnref{fn1}}
\author[1]{Shiheng \snm{Zhang}\fnref{fn1}}
\fntext[fn1]{These two authors contributed equally to this work.}

\author[1]{Jie \snm{Shen}\corref{cor1}}
\ead{shen7@purdue.edu}
\author[1,2]{Guang \snm{Lin}\corref{cor1}}
\cortext[cor1]{Corresponding authors.}
\ead{guanglin@purdue.edu}
\address[1]{Department of Mathematics, Purdue University, 150 N. University Street, West Lafayette, IN 47907-2067, USA}
\address[2]{School of Mechanical Engineering, Purdue University, 585 Purdue Mall, West Lafayette, IN 47907-2067, USA}

\begin{abstract}
Energy-Dissipative Evolutionary Deep Operator Neural Network is an operator learning neural network. It is designed to seek  numerical solutions for a class of  partial differential equations instead of a single partial differential equation, such as partial differential equations with different parameters or different initial conditions. 
The network consists of two sub-networks, the Branch net, and the Trunk net. For an objective operator $\mathcal{G}$, the Branch net encodes different input functions $u$ at the same number of sensors $y_i, i = 1,2,\cdots, m$, and the Trunk net evaluates the output function at any location. 
By minimizing the error between the evaluated output $q$ and the expected output $\mathcal{G}(u)(y)$, DeepONet generates a good approximation of the operator $\mathcal{G}$. 
In order to preserve  essential physical properties of PDEs, such as the Energy Dissipation Law, we adopt a scalar auxiliary variable approach to generate the minimization problem. It introduces a modified energy and  enables unconditional  energy dissipation law in the discrete level. By taking the parameter as a function of the time $t$ variable, this network can predict the accurate solution at any further time with feeding data only at the initial state. The data needed can be generated by the initial conditions, which are readily available. In order to validate the accuracy and efficiency of our neural networks, we provide numerical simulations of several partial differential equations, including heat equations, parametric heat equations, and Allen-Cahn equations. 
\end{abstract}



\end{frontmatter}



\section{Introduction}
Operator learning is a popular and challenging problem with potential applications across various disciplines. The opportunity to learn an operator over a domain in Euclidean spaces\cite{kovachki2021neural} and Banach spaces\cite{li2020neural} opens a new class of problems in neural network design with generalized applicability. In application to solve partial differential equations(PDEs), operator learning has the potential to predict  accurate solutions for the PDE by acquiring extensive prior knowledge \cite{khoo2021solving, bhattacharya2020model, nelsen2021random, li2020fourier, patel2021physics, opschoor2020deep, schwab2019deep, o2022derivative, wu2020data}.
In a recent paper\cite{lu2021learning}, Lu, Jin, and Karniadakis proposed an operator learning method with some deep operator networks, named as DeepONets. It is based on the universal approximation theorem \cite{cybenko1989approximation,hornik1991approximation,hornik1989multilayer}. The goal of this neural network is to learn an operator instead of a single function, which is usually the solution of a  PDE. For any operator $\mathcal{G}$ on a domain $\Omega$, we can define $\mathcal{G}$ as a mapping from $\Omega^* \rightarrow \Omega^*$ with $\mathcal{G}(u)(y) \in R$ for any $y\in \Omega$. $\mathcal{G}(u)(y)$ is the expected output of the neural network, which is usually a real number. The objective of the training is to obtain an approximation of $\mathcal{G}$, where we need to represent operators and functions in a discrete form. In practice, it is very common to represent a continuous function or operator by the values evaluated at finite and enough locations $\{x_1, x_2, \cdots, x_m \}$, which is called ‘‘sensors" in DeepONet. The network takes $[u(x_1),u(x_2),\cdots, u(x_m)]$ and $y$ as the input. The loss function is the difference between the output $q$ and the expected output $\mathcal{G}(u)(y)$. Generally, there are two kinds of DeepONet, Stacked DeepONet, and Unstacked DeepONet. The Stacked DeepONet consists of $p$ branch networks and one trunk network. The number of the Trunk networks of the Unstacked DeepONet is the same as the DeepONet, but the Unstacked DeepONet merges all the $p$ branch networks into a single one. An Unstacked DeepONet combines two sub-networks, Branch net, and Trunk net. The Branch net encodes the input function $u$ at some sensors, $\{x_i\in  \Omega \,|\, i = 1, \cdots, m\}$. The output of the Branch net consists of $p$ neurons, where each neuron can be seen as a scalar, $b_j = b_j(u(x_1), u(x_2), \cdots, u(x_m))$, $j=1,2, \cdots, p$. The Trunk net encodes some evaluation points $\{y_k \in \Omega | k = 1, \cdots, n\}$, while the output also consists of $p$ neurons and each neuron is a scalar $g_j = g_j(y_1, y_2, \cdots, y_n)$, $j=1,2, \cdots, p$. The evaluation point $y_i$ can be arbitrary in order to obtain the loss function. The number of neurons at the last layer of the Trunk net and the Branch net is the same. Hence, the output of the DeepONet can be written as an inner product of $(b_1, b_2, \cdots, b_p)$ and $(g_1, g_2, \cdots, g_p)$. In other words, the relationship between the expected output and the evaluated output is $\mathcal{G}(u)(y) \approx \sum_{j=1}^{p} b_jg_j$. The DeepONet is an application of the Universal Approximation Theorem for Operator, which is proposed by Chen $\&$ Chen \cite{392253}:
\begin{thm}[Universal Approximation Theorem for Operator]
Suppose that $\Omega_1$ is a compact set in $X$, $X$ is a Banach Space, $V$ is a compact set in $C(\Omega_1)$,  $\Omega_2$ is a compact set in  $\boldsymbol{R}^d$, $\sigma$ is a continuous non-polynomial function, $\mathcal{G}$ is a nonlinear continuous operator, which maps $v$ into $C(\Omega_2)$, then for any $\epsilon>0$, there are positive integers $M, N, m$, constants $c_{i}^{k}, \zeta_{k}, \xi_{i j}^{k} \in \boldsymbol{R}$, points $\omega_{k} \in \boldsymbol{R}^{n}, x_{j} \in K_{1}, i=1, \cdots, M$, $k=1, \cdots, N, j=1, \cdots, m$, such that
$$
\begin{gathered}
\mid \mathcal{G}(u)(y)-\sum_{k=1}^{N} \sum_{i=1}^{M} c_{i}^{k} \sigma\left(\sum_{j=1}^{m} \xi_{i j}^{k} u\left(x_{j}\right)+\theta_{i}^{k}\right) 
\cdot \sigma\left(\omega_{k} \cdot y+\zeta_{k}\right) \mid<\epsilon
\end{gathered}
$$
holds for all $u \in V$ and $y \in \Omega_2$.
\end{thm}
\begin{figure}[tbh]
     \centering
     \begin{subfigure}{0.8\textwidth}
         \centering
         \includegraphics[width=\textwidth]{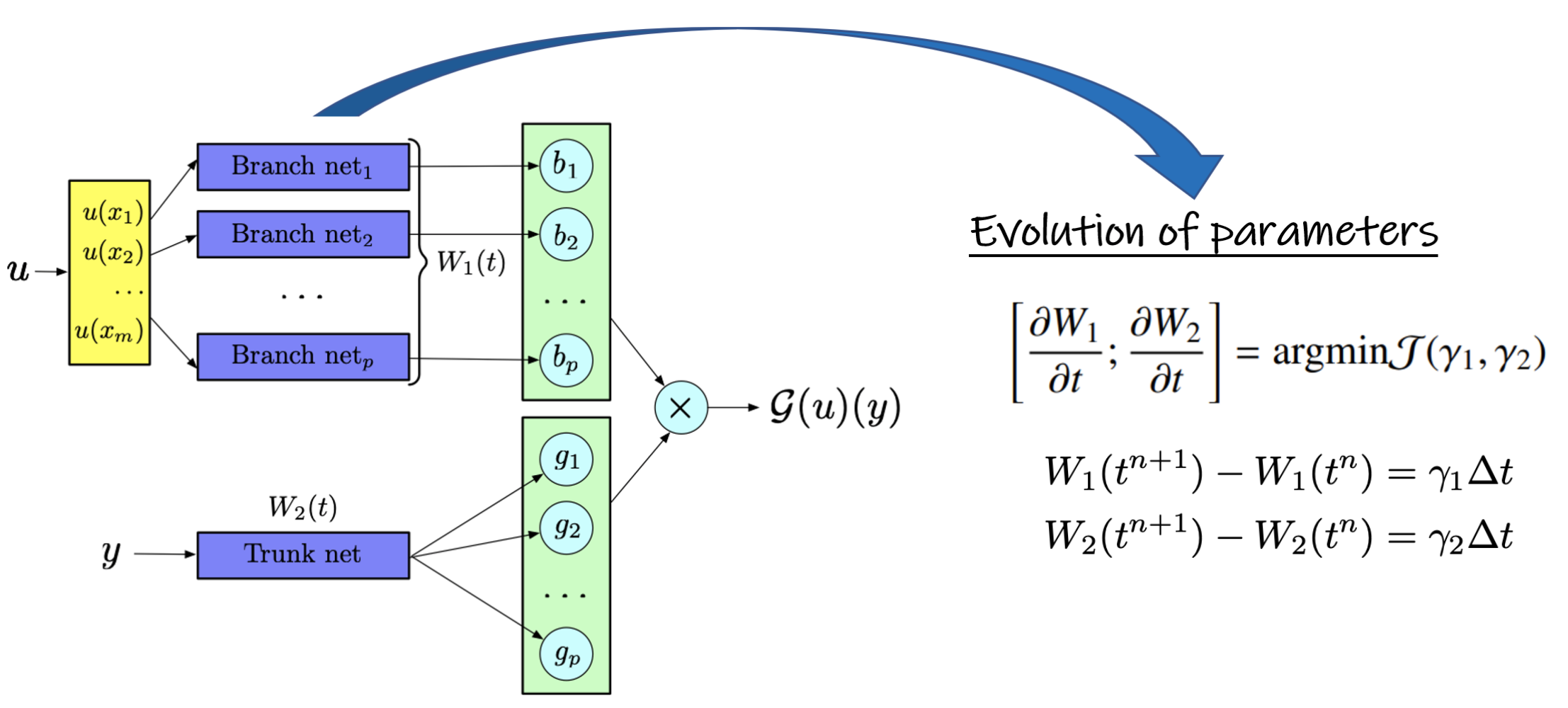}
         \caption{Stacked EDE-DeepONet}
         \label{fig:m=2, dt = 0.1}
     \end{subfigure}
     \begin{subfigure}{0.8\textwidth}
    \centering
    \includegraphics[width=\textwidth]{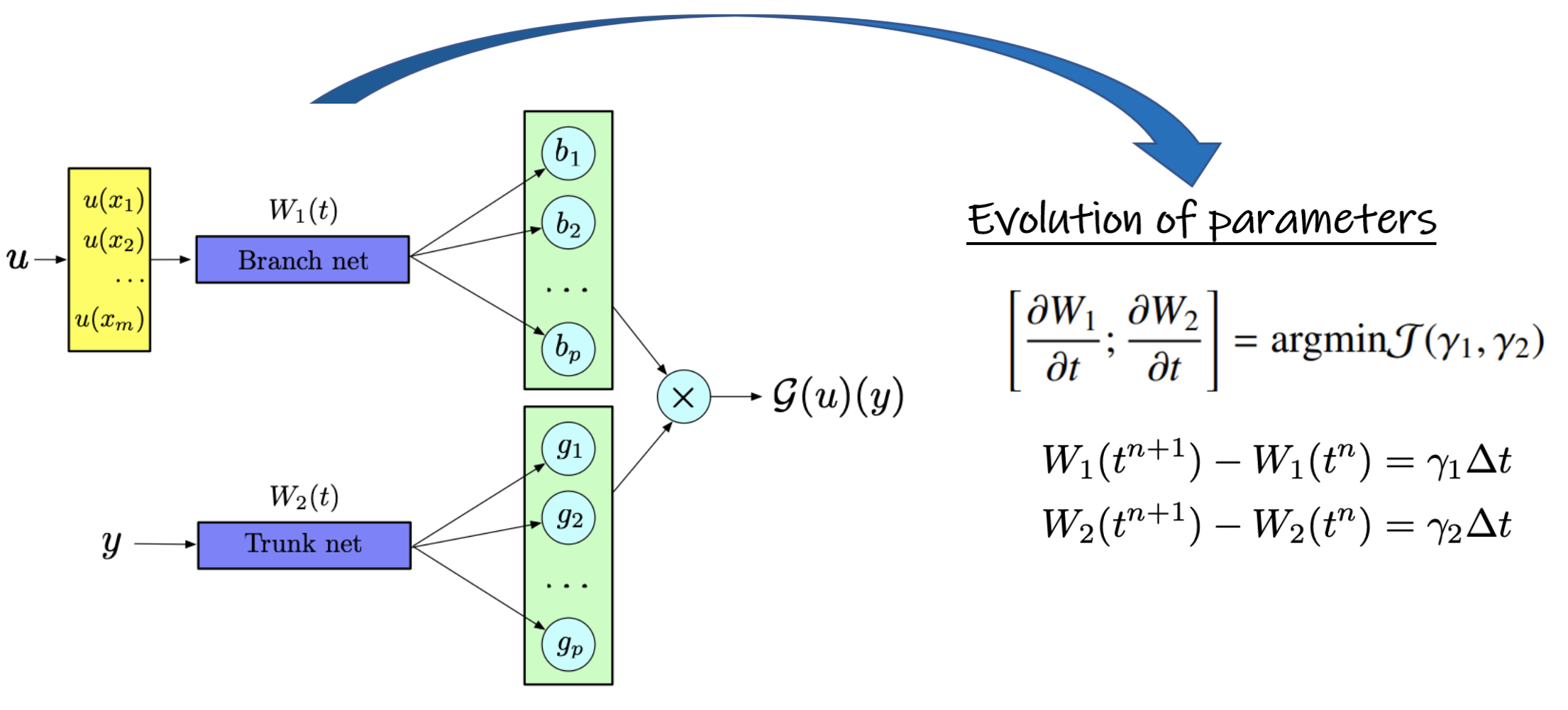}
    \caption{Unstacked EDE-DeepONet}
     \label{fig:m=2, dt = 0.1}
    \end{subfigure}
    \caption{Energy-Dissipative Evolutionary Deep Operator Neural Network. The yellow block represents input at sensors and the blue block represents subnetworks. The green blocks represent the output of the subnetworks and also the last layer of the EDE-DeepONet. The difference between the stacked and unstacked EDE-DeepONet is the number of Branch nets. In the right minimization problem, the energy term $r^2$ can be shown to be dissipative, i.e. $(r^{n+1})^2 \leq (r^{n})^2$, where $   \mathcal{J}(\gamma_1, \gamma_2 )=\frac{1}{2} \left\|\sum_{k=1}^{p}\frac{\partial {g_k(W^n_1)}}{\partial W^n_1}\gamma_1 \boldsymbol{b}_k(W^n_2) +  \sum_{k=1}^{p}g_k(W^n_1) \frac{\partial {\boldsymbol{b}_k(W^n_2)}}{\partial W^n_2} \gamma_2 -\frac{r^{n+1}}{\sqrt{E({\boldsymbol{u^n}})}} \mathcal{N}_{\boldsymbol{x}}({\boldsymbol{u^n}})\right\|_{2}^{2}$.}
\end{figure}
 For any time-dependent PDE, the training data is the form of $(u,y, \mathcal{G}(u)(y))$, where $u$ in the discrete form can be represented as $[u(x_1), u(x_2), \cdots, u(x_m)]$ in the neural network. In the original paper, they used the classic FNN\cite{raissi2018multistep} as the baseline model. For dynamic systems, various network architectures are used, including residual networks\cite{qin2021deep}, convolutional NNs(CNNs)\cite{winovich2019convpde, zhu2019physics}, recurrent NNs(RNNs)\cite{del2021learning}, neural jump stochastic differential equations\cite{jia2019neural} and neural ordinary differential equations\cite{chen2018neural}. The training performance is very promising. It predicts accurate solutions of many nonlinear ODEs and PDEs, including the simple dynamic system, gravity pendulum system, and diffusion-reaction system. However, the training data need to be generated at each time step, so it is very expensive to train the network. For a lot of initial value problems, there is no any information of $u(x,t)$ except $t=0$. It is very natural to raise a question: \textit{Can we learn an operator of a kind of time-dependent PDEs with only initial conditions?}\\
 Inspired by the Evolutionary Deep Neural Network(EDNN)\cite{du2021evolutional}, it is more convenient to learn an operator at a fixed time instead of an operator with not only spatial variables but also a time variable. With the loss of generality, we can take the time variable $t$ to be $0$ in initial value problems. Once obtained the operator at the initial time, many traditional numerical methods can be used to update the solution. More specifically, assuming that the initial condition operator has been trained well, we can consider the parameters of the Branch net and the Trunk net as a function with respect to the time variable as shown in Figure 1. More specifically, for a given initial value problem,
\begin{equation}
    \left\{
    \begin{aligned}
    &\frac{\partial u}{\partial t} = s(u) \\
    &u(x,0) = f(x), \quad x\in \omega
    \end{aligned}
    \right.
\end{equation}
the objective is to approximate the operator $\mathcal{G}: u \mapsto \mathcal{G}(u)$. The input is $([u(x_1),u(x_2), \cdots, u(x_m)],y,\mathcal{G}(u)(y))$, where $\{ x_1,x_2,\cdots,x_m\}$ are the sensors and $\mathcal{G}(u)(y) = f(y)$. The training process at the initial step is the same as the DeepONet, so we can use the same architecture to train the initial condition operator. The output of the Branch net can be written as ${\boldsymbol{b}} = \boldsymbol{b}(u(x_1,0), u(x_2,0), \cdots, u(x_{m_1},0)) =  \boldsymbol{b^{'}}(x_1, x_2, \cdots, x_{m_1}; W_1)$, where $W_1$ are the parameters in the Branch net. The output of the Trunk net can be written as ${ \boldsymbol{g}} = \boldsymbol{g}(y;W_2)$, where $W_2$ are the parameters in the
Trunk net. Once trained well, we will 
regard the parameters as a function of $t$ and $W_1$, $W_2$ as the initial conditions of $W_1(t)$ and $W_2(t)$. By the architecture of the Unstacked DeepONet, we can write the solution at initial time $t_0=0$ as 
 \begin{equation}
     u(x,t_0) \approx \sum_{j=1}^{p} b_jg_j  = \boldsymbol{b}^T  \boldsymbol{g} \text{ for any given initial condition } f(x)
 \end{equation} 
 We do not need any more data to obtain the approximation of $u(x,t_1)$. $u(x,t_1)$ should be consistent with $W_1(t_1)$ and $W_2(t_1)$. With the idea of the numerical solver for PDEs, it is easy to obtain $W_1(t_1)$ and $W_2(t_1)$ if $\frac{\partial W_1}{\partial t}$ and $\frac{\partial W_2}{\partial t}$ are known. The time derivative of the solution $u$ can be written by a chain rule:
 \begin{equation}
     \frac{\partial u}{\partial t} = \frac{\partial u}{\partial W}\frac{\partial W}{\partial t}
 \end{equation}
 where $W$ consists of $W_1$ and $W_2$. $\frac{\partial W}{\partial t}$ can be solved by a least square problem. Once we get  $\frac{\partial W}{\partial t}$, we can use any traditional time discretization schemes to get $W^{n+1}$ with $W^n$.\\
 The choice of the traditional time discretization scheme is dependent on the specific problem. The Euler or Runge–Kutta methods are commonly used in the evolutionary network. We are going to introduce a method with unconditional energy dissipation, which is the Energy-Dissipative Evolutionary Deep Operator Neural Network(EDE-DeepONet). Many kinds of PDEs are derived from basic physical laws, such as Netwon's Law, Conservation Law and Energy Dissipation Law. In many areas of science and engineering, particularly in the field of materials science, gradient flows are commonly employed in mathematical models\cite{allen1979microscopic, anderson1998diffuse, cahn1958free, doi1988theory, elder2002modeling, gurtin1996two, leslie1979theory, yue2004diffuse}. When approximating the solution of a certain PDE, it is desirable to satisfy these laws. We consider a gradient flow problem, 
 \begin{equation}
     \frac{\partial {u}}{\partial t} = - \frac{\delta E}{\delta u},
 \end{equation}
  where $E$ is a certain free energy functional. Since the general explicit Euler method does not possess the unconditionally dissipative energy dissipation law, we applied a scalar auxiliary variable(SAV) method\cite{shen2018scalar} to generate the required least square problem. It introduces a new modified energy and the unconditionally dissipative modified energy dissipation law is satisfied for each iterative step. SAV method has been applied to solve plenty of PDEs with thermodynamically consistent property. It is robust, easy to implement and accurate to predict the solution. Introducing this method to neural network helps us explore how to combine neural network models and physical laws.\\
   The objectives of this article is:
 \begin{itemize}
     \item Designing an operator learning neural network without data except the given information.
     \item Predicting solutions of parametric PDEs after a long time period.
     \item Keeping energy dissipative property of a dynamic system.
 \end{itemize}
 Our main contributions are:
 \begin{itemize}
     \item Constructing an evolutionary operator learning neural network to solve PDEs.
     \item Solving a kind of PDEs with different parameters in a single neural network.
     \item Introducing the modified energy in the neural network and applying SAV algorithm to keep the unconditionally modified energy dissipation law.
     \item Introducing an adaptive time stepping strategy and restart strategy in order to speed the training process.
 \end{itemize}
 The organization of this paper is as follows: In Section 2, we introduce the Evolutionary Deep Operator Neural Network for a given PDE problem. In Section 3, we consider the physics law behind the gradient flow problem and apply the SAV method to obtain the energy dissipation law. We proposed a new architecture for neural network, EDE-DeepONet. In Section 4, we presented two adaptive time stepping strategies, where the second one is called restart in some cases. In Section 5, we generally introduced the architecture of the EDE-DeepONet. In Section 6, we implement our neural network to predict solutions of heat equations, parametric heat equations, and Allen-Cahn equations to show the numerical results.
 \section{Evolutionary Deep Operator Neural Network}
 Consider a general gradient flow problem,
 \begin{equation}
 \begin{aligned}
    &\frac{\partial \boldsymbol{u}}{\partial t}+\mathcal{N}_{\boldsymbol{x}}(\boldsymbol{u})=0\\
    &\boldsymbol{u}(\boldsymbol{x}, 0) = \boldsymbol{f}(\boldsymbol{x})
 \end{aligned}
 \end{equation}
 where $\boldsymbol{u}\in \boldsymbol{R}^l$, $\mathcal{N}_{\boldsymbol{x}}(\boldsymbol{u})$ can be written as a variational derivative of a free energy functional $E[u(\boldsymbol{x})]$ bounded from below, $\mathcal{N}_{\boldsymbol{x}}(\boldsymbol{u}) = \frac{\delta E}{\delta \boldsymbol{u}}$. The first step is to approximate the initial condition operator with DeepONet.
 \subsection{Operator learning}
For an operator $\mathcal{G}$, $\mathcal{G}: \boldsymbol{u}(\boldsymbol{x}) \mapsto \boldsymbol{f}(\boldsymbol{x})$, the data feed into the DeepONet is in the form $(\boldsymbol{u}, y, \mathcal{G}(\boldsymbol{u})(y))$. It is obtained by the given initial conditions. The branch network takes $[\boldsymbol{u}(\boldsymbol{x_1}), \boldsymbol{u}(\boldsymbol{x_2}), \cdots, \boldsymbol{u}(\boldsymbol{x_m})]^T$ as the input, which is the numerical representation of $\boldsymbol{u}$, and $[\boldsymbol{b}_1, \boldsymbol{b}_2, \cdots, \boldsymbol{b}_p]^T \in \boldsymbol{R}^{p\times l}$, where $\boldsymbol{b}_k\in \boldsymbol{R}^l \text{ for } k=1,2, \cdots, p$, as outputs. The trunk network takes $\boldsymbol{y}$ as the input and $[g_1, g_2, \cdots, g_p] \in \boldsymbol{R}^p$ as outputs. The Unstacked DeepONet net uses FNN as the baseline model and concatenate the function value at sensor locations and the evaluated point together, i.e. $[\boldsymbol{u}(\boldsymbol{x_1}), \boldsymbol{u}(\boldsymbol{x_2}), \cdots, \boldsymbol{u}(\boldsymbol{x_m}), \boldsymbol{y}]^T$. As the equation in the Universal Approximation Theorem for Operators, we can take the product of $\boldsymbol{h}$ and $t$, then we obtain:
  \begin{equation}
      \mathcal{G}(\boldsymbol{u})({\boldsymbol{x}}) \approx \sum_{k=1}^{p} g_k\boldsymbol{b}_k
  \end{equation}
The activation functions are applied to the trunk net in the last layer. There is no bias in this network. However, according to the theorem 1, the generalization error can be reduced by adding bias. We also give the form with bias $\boldsymbol{b_0}$:
  \begin{equation}
      \mathcal{G}(\boldsymbol{u})({\boldsymbol{x}}) \approx \sum_{k=1}^{p} g_k\boldsymbol{b}_k + \boldsymbol{b_0}
  \end{equation}
  As mentioned before, we assumed the initial condition operator has been trained very well. We are going to find the update rule of the parameters to evolve the neural network.
  \subsection{The evolution of parameters in the neural network}
 Denoting the parameters in the branch network as $W_1$ and the parameters in the trunk network as $W_2$, $W_1$ and $W_2$ can be regarded a function of $t$ since they change in every time step. According to the derivative's chain rule, we have
  \begin{equation}
      \frac{\partial {\boldsymbol{u}}}{\partial t}=\frac{\partial {\boldsymbol{u}}}{\partial W_1} \frac{\partial {W_1}}{\partial t} + \frac{\partial {\boldsymbol{u}}}{\partial W_2} \frac{\partial {W_2}}{\partial t}
  \end{equation}
  Since $\boldsymbol{u} = \sum_{k=1}^{p} g_k\boldsymbol{b}_k = \sum_{k=1}^{p} g_k(W_1(t))\boldsymbol{b}_k(W_2(t))$, then
 \begin{equation}
    \begin{aligned}
        \frac{\partial {\boldsymbol{u}}}{\partial t} &=  \sum_{k=1}^{p}\frac{\partial {g_k(W_1(t))}}{\partial W_1} \frac{\partial {W_1}}{\partial t}\boldsymbol{b}_k(W_2(t)) +  \sum_{k=1}^{p}g_k(W_1(t)) \frac{\partial {\boldsymbol{b}_k(W_2(t))}}{\partial W_2} \frac{\partial {W_2}}{\partial t} 
    \end{aligned} 
 \end{equation} 
 Our objective is to obtain $\frac{\partial {W_1}}{\partial t}$ and $\frac{\partial {W_2}}{\partial t}$, the update rule for parameters. It is equivalent to solve a minimization problem,
 \begin{equation}
     \left[\frac{\partial {W_1}}{\partial t}; \frac{\partial {W_2}}{\partial t} \right]=\text{argmin} \mathcal{J}(\gamma_1, \gamma_2 )
 \end{equation}
 where
 \begin{equation}
     \mathcal{J}(\gamma_1, \gamma_2 )=\frac{1}{2} \left\|\sum_{k=1}^{p}\frac{\partial {g_k(W_1(t))}}{\partial W_1}\gamma_1 \boldsymbol{b}_k(W_2(t)) +  \sum_{k=1}^{p}g_k(W_1(t)) \frac{\partial {\boldsymbol{b}_k(W_2(t))}}{\partial W_2} \gamma_2 -\mathcal{N}_{\boldsymbol{x}}({\boldsymbol{u}})\right\|_{2}^{2} 
 \end{equation}
 In this article, the inner product $(a,b)$ is defined in the integral sense, $(a,b) = \int_{\Omega} a(\boldsymbol{x})b(\boldsymbol{x}) \mathrm{~d} \boldsymbol{x}$ and the $L_2$ norm is defined as $\norm{a}^2_2 = \int_{\Omega} |a(\boldsymbol{x})|^2 \mathrm{~d} \boldsymbol{x}$.\\
 The minimization problem can be transformed into a linear system by the first-order optimal condition:
 \begin{align}
     &\frac{\partial {\mathcal{J}}}{\partial \gamma_1} = \int_{\Omega}  \left(\sum_{k=1}^{p}\frac{\partial {g_k(W_1(t))}}{\partial W_1}\boldsymbol{b}_k(W_2(t)) \right)^T \left(\gamma_1 \sum_{k=1}^{p}\frac{\partial {g_k(W_1(t))}}{\partial W_1}\boldsymbol{b}_k(W_2(t))  + \sum_{k=1}^{p}g_k(W_1(t)) \frac{\partial {\boldsymbol{b}_k(W_2(t))}}{\partial W_2} \gamma_2 -\mathcal{N}_{\boldsymbol{x}}({\boldsymbol{u}}) \right) \mathrm{d} \boldsymbol{x}=0\\
     &\frac{\partial {\mathcal{J}}}{\partial \gamma_2} = \int_{\Omega}  \left(\sum_{k=1}^{p}g_k(W_1(t)) \frac{\partial {\boldsymbol{b}_k(W_2(t))}}{\partial W_2} \right)^T \left(\gamma_1 \sum_{k=1}^{p}\frac{\partial {g_k(W_1(t))}}{\partial W_1}\boldsymbol{b}_k(W_2(t))  + \sum_{k=1}^{p}g_k(W_1(t)) \frac{\partial {\boldsymbol{b}_k(W_2(t))}}{\partial W_2} \gamma_2 -\mathcal{N}_{\boldsymbol{x}}({\boldsymbol{u}}) \right) \mathrm{d} \boldsymbol{x}=0
 \end{align}
 In this system, the gradient with respect to $W_1(t)$ and $W_2(t)$ can be computed by automatic differentiation at each time step. By denoting 
 \begin{align}
     &(\mathbf{J_1})_{i {j_1}}=\sum_{k=1}^{p}\frac{\partial {g_k(W_1(t))}}{\partial W^{j_1}_1}\boldsymbol{b}^i_k(W_2(t))\\
     &(\mathbf{J_2})_{i {j_2}}=\sum_{k=1}^{p}g_k(W_1(t)) \frac{\partial {\boldsymbol{b}^i_k(W_2(t))}}{\partial W^{j_2}_2}\\
     &(\mathbf{N})_{i}=\mathcal{N}\left(\boldsymbol{u}_{\boldsymbol{x}}^{i}\right)
 \end{align}
 where $i=1,2,\cdots,l$, $j_1=1,2,\cdots, N^b_{\text{para}}$, $j_2=1,2,\cdots, N^t_{\text{para}}$. $N^b_{\text{para}}$ is the number of parameters in Branch net and  $N^t_{\text{para}}$ is the number of parameters in Trunk net. $\mathbf{N}$ is generated by the DeepONet, so it can be evaluated at any spatial point. The above integrals can be approximated by numerical methods:
 \begin{align}
     &\frac{1}{|\Omega|}\int_{\Omega}  \left(\sum_{k=1}^{p}\frac{\partial {g_k(W_1(t))}}{\partial W_1}\boldsymbol{b}_k(W_2(t)) \right)^T \left( \sum_{k=1}^{p}\frac{\partial {g_k(W_1(t))}}{\partial W_1}\boldsymbol{b}_k(W_2(t)) \right) \mathrm{d} \boldsymbol{x} = \lim_{l \rightarrow \infty} \frac{1}{l}\mathbf{J^T_1}\mathbf{J_1}\\
     &\frac{1}{|\Omega|} \int_{\Omega}  \left(\sum_{k=1}^{p}g_k(W_1(t)) \frac{\partial {\boldsymbol{b}_k(W_2(t))}}{\partial W_2} \right)^T \left(  \sum_{k=1}^{p}g_k(W_1(t)) \frac{\partial {\boldsymbol{b}_k(W_2(t))}}{\partial W_2} \right) \mathrm{d} \boldsymbol{x} = \lim_{l \rightarrow \infty} \frac{1}{l}\mathbf{J^T_2}\mathbf{J_2}\\
     &\frac{1}{|\Omega|}\int_{\Omega}  \left(\sum_{k=1}^{p}\frac{\partial {g_k(W_1(t))}}{\partial W_1}\boldsymbol{b}_k(W_2(t)) \right)^T \left( \mathcal{N}_{\boldsymbol{x}}({\boldsymbol{u}}) \right) \mathrm{d} \boldsymbol{x} =  \lim_{l \rightarrow \infty} \frac{1}{l}\mathbf{J^T_1}\mathbf{N}
 \end{align}
 By denoting $\gamma_i^{o p t}$ as optimal values of $\gamma_i$, $i=1,2$, the objective function can be reduced to 
 \begin{align}
     &\mathbf{J^T_1}\left(\gamma_1^{o p t}\mathbf{J_1}  + \gamma_2^{o p t}\mathbf{J_2} -\mathbf{N} \right)  = 0\\
     &\mathbf{J^T_2}\left(\gamma_1^{o p t}\mathbf{J_1}  + \gamma_2^{o p t}\mathbf{J_2} -\mathbf{N} \right)  = 0
 \end{align}
The feasible solutions of the above equations are the approximated time derivatives of $W_1$ and $W_2$.
 \begin{align}
     &\frac{dW_1}{dt}  = \gamma_1^{o p t}\\
     &\frac{dW_2}{dt}  = \gamma_2^{o p t}
 \end{align}
 where the initial conditions $W_1^0$ and $W_2^0$ can be determined by DeepONets for initial condition operators. The two ODEs are the updated rules in the neural networks. The simple way to solve them is the explicit Euler method.
  \begin{align}
     &\frac{W^{n+1}_1 - W^{n}_1}{\Delta t}  = \gamma_1^{o p t}\\
     &\frac{W^{n+1}_2 - W^{n}_2}{\Delta t}  = \gamma_2^{o p t}
 \end{align}
 The neural network can calculate the solution of given PDEs at any time step $t_n$ and spatial point $\boldsymbol{x}_i$ by weights $W_1^n$ , $W_2^n$, spatial points $\boldsymbol{x}$ and initial condition ${\boldsymbol{u}}(\boldsymbol{x})$.
 
 \section{Energy Dissipative Evolutionary Deep Operator Neural Network}
 Let's reconsider the given problem. 
 \begin{equation}
 \begin{aligned}
    &\frac{\partial \boldsymbol{u}}{\partial t}+\mathcal{N}_{\boldsymbol{x}}(\boldsymbol{u})=0\\
    &\boldsymbol{u}(\boldsymbol{x}, 0) = \boldsymbol{f}(\boldsymbol{x})
 \end{aligned}
 \end{equation}
 where $\boldsymbol{u}\in \boldsymbol{R}^l$, $\mathcal{N}_{\boldsymbol{x}}(\boldsymbol{u})$ can be written as a variational derivative of a free energy functional $E[\boldsymbol{u}(\boldsymbol{x})]$ bounded from below, $\mathcal{N}_{\boldsymbol{x}}(\boldsymbol{u}) = \frac{\delta E}{\delta \boldsymbol{u}}$. Taking the inner product with $\mathcal{N}_{\boldsymbol{x}}(\boldsymbol{u})$ of the first equation, we obtain the energy dissipation property
  \begin{equation}
     \frac{dE[\boldsymbol{u}(\boldsymbol{x})]}{dt} = \left(\frac{\delta E}{\delta \boldsymbol{u}}, \frac{\partial \boldsymbol{u}}{\partial t}\right) = \left(\mathcal{N}_{\boldsymbol{x}}(\boldsymbol{u}), \frac{\partial \boldsymbol{u}}{\partial t}\right)= -\left(\mathcal{N}_{\boldsymbol{x}}(\boldsymbol{u}), \mathcal{N}_{\boldsymbol{x}}(\boldsymbol{u})\right) \leq 0
 \end{equation}
 However, it is usually hard for a numerical algorithm to be efficient as well as energy dissipative. Recently, the SAV approach \cite{shen2018scalar} was introduced to construct numerical schemes  which is energy dissipative (with a modified energy), accurate, robust and easy to implement. More precisely,  assuming  $E[\boldsymbol{u}(\boldsymbol{x})]>0$, it introduces a $r(t) = \sqrt{E[\boldsymbol{u}(\boldsymbol{x},t)]}$, and expands the gradient flow problem  as
 
\begin{equation}\label{expanded}
    \begin{aligned}
        &\frac{\partial \boldsymbol{u}}{\partial t}= - \frac{r}{\sqrt{E(\boldsymbol{u})}} \mathcal{N}_{\boldsymbol{x}}\left(\boldsymbol{u}\right)\\
        &r_{t}=\frac{1}{2 \sqrt{{E}(\boldsymbol{u})}} \left( \mathcal{N}_{\boldsymbol{x}}\left(\boldsymbol{u}\right), \frac{\partial \boldsymbol{u}}{\partial t} \right)
    \end{aligned}
\end{equation}
With $r(0)=\sqrt{E[\boldsymbol{u}(\boldsymbol{x},t)]}$, the above system has a solution $r(t)\equiv \sqrt{E[\boldsymbol{u}(\boldsymbol{x},t)]}$ and $\boldsymbol{u}$ being the solution of the original problem.

\subsection{First order scheme}
By setting $\boldsymbol{u}^{n} = \sum_{k=1}^{p} g_k\boldsymbol{b}_k$, a first order  scheme can be constructed as
 \begin{equation}
 \begin{aligned}
    &\frac{\boldsymbol{u}^{n+1} -  \boldsymbol{u}^{n}}{\Delta t} = -\frac{r^{n+1}}{\sqrt{E(\boldsymbol{u}^{n})}} \mathcal{N}_{\boldsymbol{x}}(\boldsymbol{u}^{n})\\
    &\frac{r^{n+1} - r^n}{\Delta t} = \frac{1}{2\sqrt{E(\boldsymbol{u}^{n})}}\int_{\Omega} \mathcal{N}_{\boldsymbol{x}}(\boldsymbol{u}^{n})\frac{\boldsymbol{u}^{n+1} - \boldsymbol{u}^{n}}{\Delta t} dx.
\end{aligned}
\end{equation}
  This is a coupled system of equations for $(r^{n+1},\boldsymbol{u}^{n+1})$. But it can be easily  decoupled as follows. Plugging the first equation into the second one, we obtain:
  \begin{equation}
      \frac{r^{n+1} - r^n}{\Delta t} = -\frac{r^{n+1}}{2{E(\boldsymbol{u}^{n})}}  \norm{\mathcal{N}_{\boldsymbol{x}}(\boldsymbol{u}^{n})}^2,
  \end{equation}
which implies
   \begin{equation}
      r^{n+1} =\left(1+\frac{\Delta t }{2{E(\boldsymbol{u}^{n})}}  \norm{\mathcal{N}_{\boldsymbol{x}}(\boldsymbol{u}^{n})}^2 \right)^{-1} r^n 
  \end{equation}
 \begin{thm}[Discrete Energy Dissipation Law]
    With the modified energy define above, the scheme is unconditionally energy stable, i.e.
    \begin{equation}
        (r^{n+1})^2 - (r^{n})^2 \leq 0.
    \end{equation}
 \end{thm}
 \begin{proof}
 Taking the inner product of the first equation with $\frac{r^{n+1}}{\sqrt{E(\boldsymbol{u}^n)}} \mathcal{N}_{\boldsymbol{x}}(\boldsymbol{u}^n)$ and the second equation with $2r^{n+1}$
 \begin{equation}
     \begin{aligned}
        (r^{n+1})^2 - (r^{n})^2 &= 2r^{n+1}(r^{n+1} - r^{n})- (r^{n+1} - r^{n})^2\\
        &=\frac{\Delta tr^{n+1}}{\sqrt{E(\boldsymbol{u}^n)}}\int_{\Omega} \mathcal{N}_{\boldsymbol{x}}(\boldsymbol{u}^n)\frac{\boldsymbol{u}^{n+1} - \boldsymbol{u}^n}{\Delta t} dx- (r^{n+1} - r^{n})^2\\
        &=-\left(\frac{r^{n+1}}{\sqrt{E(\boldsymbol{u}^n)}}\right)^2\int_{\Omega} \mathcal{N}_{\boldsymbol{x}}(\boldsymbol{u}^n) \mathcal{N}_{\boldsymbol{x}}(\boldsymbol{u}^n) dx- (r^{n+1} - r^{n})^2\\
        &\leq 0
     \end{aligned}
 \end{equation}
 \end{proof}
 
In order to maintain the modified energy dissipation law in the evolution neural network, we only need to replace $\mathcal{N}_{\boldsymbol{x}}({\boldsymbol{u}})$ by $\frac{r^{n+1}}{\sqrt{E({\boldsymbol{u}})}} \mathcal{N}_{\boldsymbol{x}}({\boldsymbol{u}})$ in section 2. The update rule of the neural network is
\begin{equation}
     \left[\frac{\partial {W_1}}{\partial t}; \frac{\partial {W_2}}{\partial t} \right]=\text{argmin} \mathcal{J}(\gamma_1, \gamma_2 )
 \end{equation}
 where
 \begin{equation}
     \mathcal{J}(\gamma_1, \gamma_2 )=\frac{1}{2} \left\|\sum_{k=1}^{p}\frac{\partial {g_k(W^n_1)}}{\partial W^n_1}\gamma_1 \boldsymbol{b}_k(W^n_2) +  \sum_{k=1}^{p}g_k(W^n_1) \frac{\partial {\boldsymbol{b}_k(W^n_2)}}{\partial W^n_2} \gamma_2 -\frac{r^{n+1}}{\sqrt{E({\boldsymbol{u^n}})}} \mathcal{N}_{\boldsymbol{x}}({\boldsymbol{u^n}})\right\|_{2}^{2}
 \end{equation}
 The corresponding linear system of the first order optimal condition is 
 \begin{align}
     &\mathbf{J^T_1}\left(\gamma_1^{o p t}\mathbf{J_1}  + \gamma_2^{o p t}\mathbf{J_2} -\frac{r^{n+1}}{\sqrt{E({\boldsymbol{u^n}})}} \mathbf{N} \right)  = 0\\
     &\mathbf{J^T_2}\left(\gamma_1^{o p t}\mathbf{J_1}  + \gamma_2^{o p t}\mathbf{J_2} -\frac{r^{n+1}}{\sqrt{E({\boldsymbol{u^n}})}} \mathbf{N} \right)  = 0
 \end{align}
 where
 \begin{align}
     &(\mathbf{J_1})_{i {j_1}}=\sum_{k=1}^{p}\frac{\partial {g_k(W_1^n)}}{\partial W^{n,j_1}_1}\boldsymbol{b}^i_k(W_2^n)\\
     &(\mathbf{J_2})_{i {j_2}}=\sum_{k=1}^{p}g_k(W^n_1) \frac{\partial {\boldsymbol{b}^i_k(W^n_2)}}{\partial W^{n, j_2}_2}\\
     &(\mathbf{N})_{i}=\mathcal{N}\left(\boldsymbol{u}_{\boldsymbol{x}}^{i}\right)
 \end{align}
 and $i=1,2,\cdots,l$, $j_1=1,2,\cdots, N^b_{\text{para}}$, $j_2=1,2,\cdots, N^t_{\text{para}}$. $N^b_{\text{para}}$ is the number of parameters in Branch net and  $N^t_{\text{para}}$ is the number of parameters in Trunk net.
 After getting $\gamma_1^{o p t}$ and $\gamma_2^{o p t}$, $W^{n+1}$ can be obtained by the Forward Euler method as equation (24) and (25).
   \begin{align}
     &W^{n+1}_1   =W^{n}_1 +  \gamma_1^{o p t} {\Delta t}\\
     &W^{n+1}_2   =W^{n}_2 +  \gamma_2^{o p t} {\Delta t}
 \end{align}
 
 \section{Adaptive time stepping strategy and Restart strategy}
 One of the advantages of an unconditionally stable scheme is that the adaptive time step can be utilized. Since the coefficient of $N_x$, $\frac{r^{n+1}}{\sqrt{E^n}}$ should be around 1, by denoting $\xi^{n+1} = \frac{r^{n+1}}{\sqrt{E^n}}$,  larger $\Delta t$ is allowed when $\xi$ is close to $1$ and the smaller $\Delta t$ is needed when $\xi$ is far away from $1$. Thus, a simple  adaptive time-stepping strategy can be described as follows:\\
 \begin{algorithm}[H]
 \caption{Adaptive time stepping strategy}
 \begin{algorithmic}
     \item 1. Set the tolerance for $\xi$ as $\epsilon_0$ and $\epsilon_1$, the initial time step $\Delta t$, the maximum time step $\Delta t_{max}$ and the minimum time step $\Delta t_{min}$
     \item 2. Compute $u^{n+1}$.
     \item 3. Compute $\xi^{n+1} = \frac{r^{n+1}}{\sqrt{E^n}}$.
     \item 4. \textbf{If} $|1-\xi^{n+1}|>\epsilon_0$,\\
     \quad \quad \textbf{Then} $\Delta t = \max(\Delta t_{min}, \Delta t /2)$;\\
     \quad \textbf{Else if} $|1-\xi^{n+1}|<\epsilon_1$,\\
     \quad \quad \textbf{Then} $\Delta t = \min(\Delta t_{max}, 2\Delta t )$.\\
     \quad \textbf{Go to Step 2}.
     \item 5. Update time step $\Delta t$.
 \end{algorithmic}
 \end{algorithm}
 Another popular strategy to keep $r$ approximating the original energy $E$ is to reset the SAV $r^{n+1}$ to be $E^{n+1}$ in some scenarios. The specific algorithm is as following:\\
 \begin{algorithm}[H]
 \caption{Restart strategy}
 \begin{algorithmic}
     \item 1. Set the tolerance for $\xi$ as $\epsilon_2$.
     \item 2. Compute $u^{n+1}$.
     \item 3. Compute $\xi^{n+1} = \frac{r^{n+1}}{\sqrt{E^n}}$.
     \item 4. \textbf{If} $|1-\xi^{n+1}|>\epsilon_2$,\\
     \quad \quad \textbf{Then} $r^{n+1} = \sqrt{E^{n+1}}$ and 
     \textbf{Go to Step 2}.
     \item 5. Go to next iteration.
 \end{algorithmic}
 \end{algorithm}
 The choice for $\epsilon_0$, $\epsilon_1$ should be some small tolerance, usually $10^{-1}$ and $10^{-3}$. The choices for $\Delta t_{max}$ and $\Delta t_{min}$ are quite dependent on $\Delta t$, usually $\Delta t_{max} = 10^3 \times \Delta t$ and $\Delta t_{min} = 10^{-3}\times \Delta t$. In Algorithm 2, we usually take $\epsilon_2$ as $2\times 10^{-2}$.
 
 \section{Algorithm for EDE-DeepONet}
 A general approach to solving a time-dependent PDE with EDE-DeepONet can be summarized in Algorithm 3.
\begin{algorithm}[H]
\caption{Energy-Dissipative Evolutionary Deep Operator Neural Networks(EDE-DeepONet)}
\begin{algorithmic}
    \item 1. Generate input data samples in the form of $(u,y,\mathcal{G}(u)(y))$ for the DeepONet, where $\mathcal{G}$ is the objective operator. Each specific input function $u$ can be generated in the same sensor locations $\{x_1, x_2, \cdots,x_m \}$.
    \item 2. Feed $[u(x_{1}),u(x_{2}), \cdots, u(x_{m})])$ into the branch network and $y \in Y$ into the trunk network. Denote the output of the DeepONet as $q$.
    \item 3. Update the parameters in the DeepONet by minimizing a cost function, where the cost function can be taken as the mean squared error as $\frac{1}{|Y|}\sum_{y\in Y} \norm{\mathcal{G}(u)(y)-q}^2$.
    \item 4. Once the DeepONet has been trained well, solve the system of equations of (36) and (37) to obtain $\left[\frac{\partial {W_1}}{\partial t}; \frac{\partial {W_2}}{\partial t} \right]$.
    \item 5.The value of $\left[\frac{\partial {W_1}}{\partial t}; \frac{\partial {W_2}}{\partial t} \right]$ can be obtained in the current step. Since the parameters $W_1^n$ in the branch network, and $W_2^n$ in the trunk network are known, $W_1^{n+1}$ and $W_2^{n+1}$ for the next step can be also obtained by the Forward Euler method or Runge-Kutta method.
    \item 6. Repeat step 5 until the final time $T$, where $T=t_0 + s\Delta t$, $t_0$ is the initial time of the given PDE, $\Delta t$ is the time step in step 5 and $s$ is the number of repeated times of step 5.  
    \item 7. Output the solution at time $T$ in the DeepONet with parameters obtained in step 6.
\end{algorithmic}
\end{algorithm}

\section{Numerical Experiments}
In this section, we implement EDE-DeepONet to solve heat equations, parametric heat equations, and Allen-Cahn equations to show its performance and accuracy. 
\subsection{Example 1: Simple heat equations}
To show the accuracy of the EDE-DeepONet, we start with the simple heat equation with different initial conditions since we already have the exact solution.
A 1D heat equation system can be described by
\begin{align}
    &u_t =  u_{xx}&\\
    &u(x,0) = f&\\
    &u(0,t) = u(2,t) = 0&
\end{align}
By the method of separation of variables, we can derive the solution to the heat equation.
If we set $f(x) = a sin(\pi x)$, the solution is $u(x,t) = a sin(\pi x) e^{ - \pi^2 t}$, where $a\in[1,2]$. The corresponding energy is $E(u) = \int_{0}^{2} \frac{1}{2} |u_x|^2 dx \approx \Delta x(\sum_{i=1}^{n} \frac{1}{2} |u_x(x_i)|^2)$. With different parameters $a$, the above equation describes a kind of PDE. The input data samples can be generated as $(a, x, \mathcal{G}(a)(x))$, where $\mathcal{G}(a)(x) = a\sin(\pi x)$ for specific $a$ and $x$. When generating the initial data samples, we choose 50 points from $[0,2)$ uniformly for x and 50 random values of $a$ from $[1,2]$. The time step when updating the parameters in the neural network is $2.5\times 10^4$. The number of iteration steps is 400. We compared the different solutions with 4 different $a$, $1.0$, $1.5$, $1.8$, $2.5$ every 100 steps. Although $a=2.5$ is out of the range of training data, it still performs well in this model. With the exact solution, we also get the error with different $a$ as Table 1. The error is defined by $\frac{1}{N_x}\sum_{k=1}^{N_x}({u(x_k) - \hat{u}(x_k)})^2,$ where $N_x = 51$, $u$ is the solution obtained by EDE-DeepONet and $\hat{u}$ is the exact solution. To illustrate the relationship between the modified energy and the original energy, we compare $r^2$ and $E$ at each step as Figure 2. Both energy are actually disspative in the EDE-DeepONet except when restart strategy applied. The restart strategy is used to keep $r^2$ approaching $E$. The modified energy is initialized when the restart strategy applied. The restart strategy was triggered 
on the 370th step since the modified energy and the original energy are offset. After that, they are on the same trajectory again. It is clear that the modified energy approaches the original energy before and after the restart strategy applied. In Figure 3, we give the comparison between the exact solution and the solution obtained by EDE-DeepONet. From this simple heat equation, we show that EDE-DeepONet correctly predicts the solution of the PDE. The most important fact is that EDE-DeepONet can not only predict the solution in the training subset range but also the solution out of the training range. For instance, we take $a=2.5$ while $a \in [1,2]$ in the training process. EDE-DeepONet shows good accuracy compared to the exact solution as Figure 3 (a)-(d) and Table 1.
\begin{figure}[tbh]
     \centering
    \includegraphics[width=0.8\textwidth]{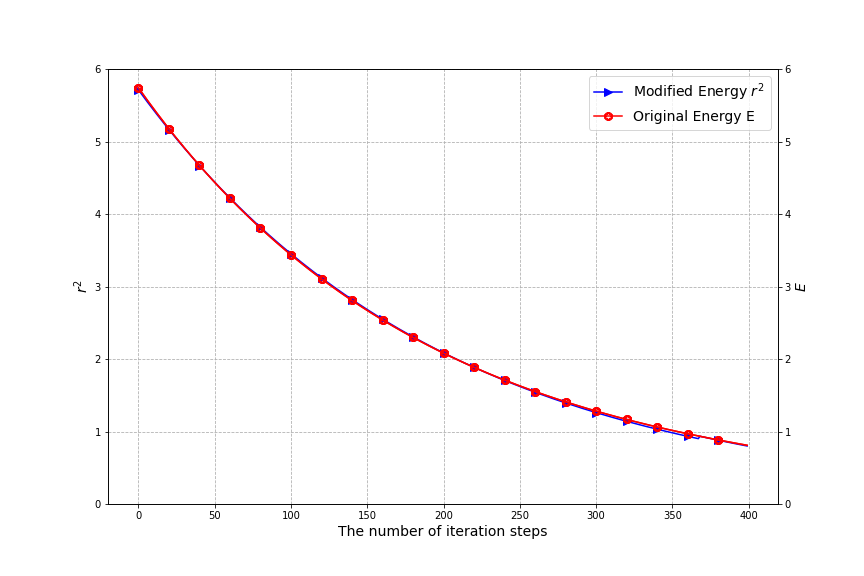}
    \label{fig:m=2, dt = 0.1}
    \caption{The heat equation: The modified energy and original energy when training the network. Each iteration step represents one forward step of the PDE's numerical solution with $\Delta t = 2.5 \times 10^{-4}$. In the EDE-DeepONet, both energy are actually dissipative, with the exception of the restart strategy. In order to maintain that $r^2$ approaches $E$, the restart strategy is employed, and the modified energy is initialized on the 370th step. The modified energy and the original energy on the same trajectory before and after the 370th step.}
\end{figure}
\begin{figure}[tbh]
     \centering
     \begin{subfigure}{0.4\textwidth}
         \centering
         \includegraphics[width=\textwidth]{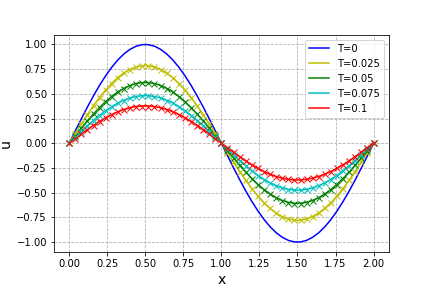}
         \caption{Initial $f(x) = sin(\pi x)$}
         \label{fig:m=2, dt = 0.1}
     \end{subfigure}
     \begin{subfigure}{0.4\textwidth}
         \centering
         \includegraphics[width=\textwidth]{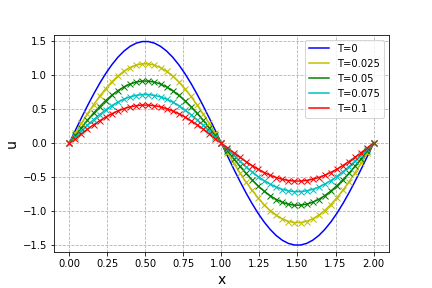}
         \caption{$f(x) = 1.5sin(\pi x)$}
         \label{fig:m=2, dt = 0.05}
     \end{subfigure}
     \begin{subfigure}{0.4\textwidth}
    \centering
    \includegraphics[width=\textwidth]{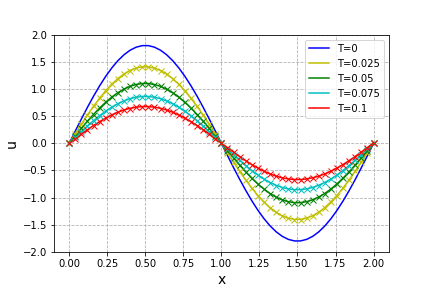}
    \caption{$f(x) = 1.8sin(\pi x)$}
     \label{fig:m=2, dt = 0.1}
    \end{subfigure}
    \begin{subfigure}{0.4\textwidth}
    \centering
    \includegraphics[width=\textwidth]{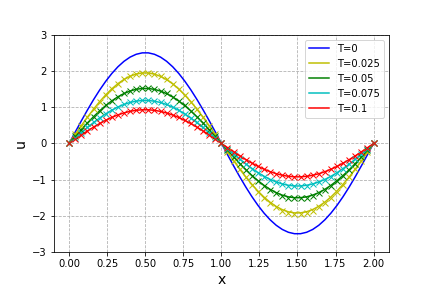}
    \caption{$f(x) = 2.5sin(\pi x)$}
     \label{fig:m=2, dt = 0.1}
    \end{subfigure}
    \caption{The heat equation: The solution with 4 different initial conditions $f(x)=a\sin(\pi x)$. The curve represents the solution obtained by the EDE-DeepONet, and xxx represents the reference solution. The training parameter $a$ is in the range of $[1,2)$, so we give three examples in this range. We also present the case out of the range. It also shows accuracy in Figure 3-(d).}
\end{figure}
\begin{table}[tbh]
\centering
    $\begin{array}{||c|ccccc||}
\hline
\text { Error } & & T = 0.025 & T = 0.05 & T = 0.075 & T = 0.1 \\
\hline \hline a = 1.0  &  & 1.47 \times 10^{-5} & 1.33 \times 10^{-5} & 1.32 \times 10^{-5} & 1.29 \times 10^{-5} \\
\hline a = 1.5 &  & 5.11 \times 10^{-6} & 7.05 \times 10^{-6} & 8.48 \times 10^{-6} & 9.81 \times 10^{-6} \\
\hline a = 1.8 &  & 1.46 \times 10^{-5} & 1.69 \times 10^{-5} & 1.79 \times 10^{-5} & 1.83 \times 10^{-5} \\
\hline a = 2.5 &  & 2.20 \times 10^{-4} & 1.34 \times 10^{-4} & 6.02 \times 10^{-5} & 1.72 \times 10^{-5} \\
\hline
\end{array}$
\caption{The heat equation: The initial condition of the PDE is $f(x) = a\sin{(\pi x)}$. The error is defined by $\frac{1}{N_x}\sum_{k=1}^{N_x}({u(x_k) - \hat{u}(x_k)})^2,$ where $N_x = 51$, $u$ is the solution obtained by EDE-DeepONet and $\hat{u}$ is the exact solution.}
\end{table}

\subsection{Example 2: Parametric heat equations}
In example 1, we take different initial conditions as our inputs. In example 2, we are going to deal with the parametric heat equations. 
A general parametric heat equation in 1D can be described by\\
\begin{align}
    &u_t = c u_{xx}&\\
    &u(x,0) = sin(\pi x)&\\
    &u(0,t) = u(2,t) = 0&
\end{align}
This PDE is more complex than the PDE in Example 1 since the parameter is inside the equation. The traditional numerical scheme needs to be run multiple times to deal with the case with different parameters because they are actually different equations. However, we only need to train the EDE-DeepONet once. The training range of $c$ is chosen as $[1,2)$. We choose 50 points of $x$ and $c$ in the same way as example 1. First, we compared the modified energy with the original energy as Figure 4. The energy is not the same as the first example since the energy depends on the parameter $c$. We compute the average of the energy with different $c$ to represent the energy of the system. This case is more complex than the first one, so it needs more restarts during the training. Even though the modified energy oscillates when restart strategy used, it keeps decreasing after each restart. Second, we give the error between the solution obtained by the EDE-DeepONet and the reference solution in Table 2, where the reference solution can be obtained explicitly by variable separation method and the error is defined in the same way as example 1. Third, we give the comparison between our solution and the reference solution in Figure 5. Same as example 1, we give the predicted solution of $c \notin [1,2]$. All of them show the good accuracy. Hence, EDE-DeepONet can actually solve parametric PDEs.
\begin{figure}[tbh]
    \centering
    \includegraphics[width=0.8\textwidth]{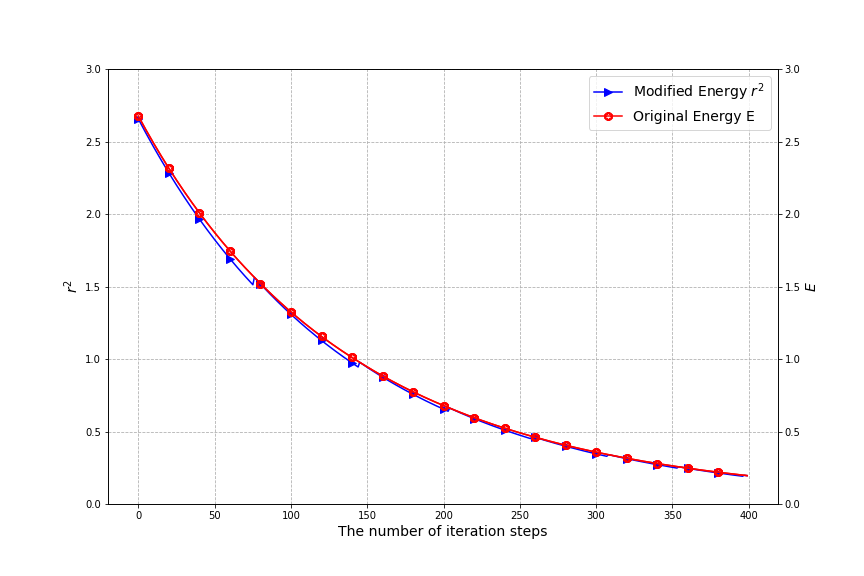}
     \label{fig:m=2, dt = 0.1}
    \caption{The parametric heat equation: The modified energy and original energy when training the network. Each iteration step represents one forward step of the PDE's numerical solution with $\Delta t = 2.5 \times 10^{-4}$. This kind of PDEs is more complicated, so it need more restarts in the training process. The original energy keeps decreasing and the modified energy also shows good approximation of the original energy.}
\end{figure}
\begin{figure}[tbh]
     \centering
     \begin{subfigure}{0.4\textwidth}
         \centering
         \includegraphics[width=\textwidth]{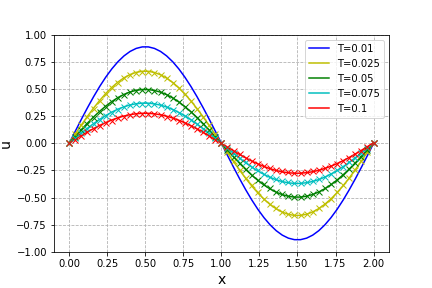}
         \caption{$c=1.2$}
         \label{fig:m=2, dt = 0.1}
     \end{subfigure}
     \begin{subfigure}{0.4\textwidth}
         \centering
         \includegraphics[width=\textwidth]{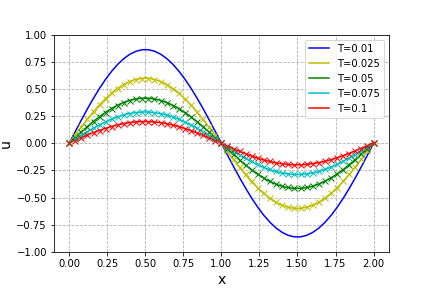}
         \caption{$c=1.5$}
         \label{fig:m=2, dt = 0.05}
     \end{subfigure}
     \begin{subfigure}{0.4\textwidth}
    \centering
    \includegraphics[width=\textwidth]{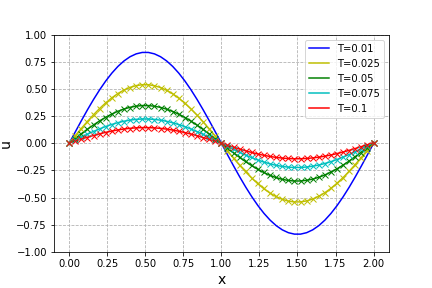}
    \caption{$c=1.8$}
     \label{fig:m=2, dt = 0.1}
    \end{subfigure}
    \begin{subfigure}{0.4\textwidth}
    \centering
    \includegraphics[width=\textwidth]{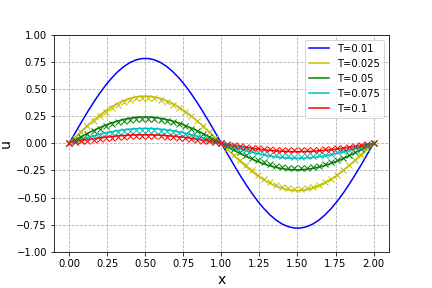}
    \caption{$c=2.5$}
     \label{fig:m=2, dt = 0.1}
    \end{subfigure}
    \caption{The parametric heat equation: The solution with 4 different parameters $c$. The curve represents the solution obtained by the EDE-DeepONet and xxx represents the reference solution. The training parameter $c$ is in the range of $[1,2)$, so we give 3 examples in this range. We also present the case out of the range in Figure 5-(d).}
\end{figure}
\begin{table}[tbh]
\centering
    $\begin{array}{||c|ccccc||}
\hline
\text { Error } & & T = 0.025 & T = 0.05 & T = 0.075 & T = 0.1 \\
\hline \hline c = 1.2  &  & 1.30 \times 10^{-5} & 1.43 \times 10^{-5} & 1.35 \times 10^{-5} & 1.20 \times 10^{-5} \\
\hline c = 1.5 &  & 1.35 \times 10^{-5} & 1.27 \times 10^{-5} & 9.80 \times 10^{-6} & 7.80 \times 10^{-6} \\
\hline c = 1.8 &  & 1.17 \times 10^{-5} & 1.03 \times 10^{-5} & 7.88 \times 10^{-5} & 1.83 \times 10^{-5} \\
\hline c = 2.5 &  & 2.20 \times 10^{-4} & 1.34 \times 10^{-4} & 6.02 \times 10^{-5} & 7.08 \times 10^{-6} \\
\hline
\end{array}$
\caption{The parametric heat equation: The initial condition of the PDE is $f(x) = \sin{(\pi x)}$. The error is defined by $\frac{1}{N_x}\sum_{k=1}^{N_x}({u(x_k) - \hat{u}(x_k)})^2$, where $N_x = 51$, $u$ is the solution obtained by EDE-DeepONet and $\hat{u}$ is the exact solution.}
\end{table}
\subsection{Example 3: Allen-Cahn equations}
The energy in Examples 1 and 2 is quadratic and the right-hand side of the PDE is linear with respect to $u$. We are going to show the result for the PDE with more complicated energy. The Allen-Cahn equation is a kind of reaction-diffusion equation. It is derived to describe the process of the phase separation. It was developed to solve a problem in the material science area and has been used to represent the moving interfaces in a phase-field model in fluid dynamics. The Allen-Cahn equation can be treated as a gradient flow in $L^2$ with some specific energy. We discussed the 1D case and 2D case as follows:
\subsubsection{1D case}
\noindent (a) {Various initial conditions:}\\
We start with the simple case, 1D Allen-Cahn equation.
It can be described by the following equations: 
\begin{align}
    &u_t =  u_{xx} - g(x)&\\
    &u(x,0) = a\sin{\pi x}&\\
    &u(-1,t) = u(1,t) = 0&
\end{align}
The corresponding  Ginzburg–Landau free energy $E[u] = \int_{0}^{1} \frac{1}{2} |u_x|^2 dx + \int_{x=0}^{x=1}  G(u) dx$, where $ G(u) = \frac{1}{4\epsilon^2}(u^2 - 1)^2$ and $g(u) = G'(u) = \frac{1}{\epsilon^2}u(u^2 - 1)$, $\epsilon = 0.1$. The parameter $\epsilon$ affects the width of the jump when arriving at the steady state as the Figure 7 (c), (j), (o) and (t). In the EDE-DeepONet, we set $\Delta t=10^{-4}$, the number of spatial points $N_x$ is 51 and the range of $a$ is $[0.1, 0.5]$. We also compared the modified energy and the original energy as Figure 6. The modified energy can approximate well to the original energy even in a much more complicated form. Then, we compared 4 different solutions with different $a\in [0.1, 0.5]$ obtained by the EDE-DeepONet and the reference solution obtained by the SAV method in traditional numerical computation as Figure 7. The error is shown in Table 3, where error is defined in the same way as example 1. $a=0.6 \notin [0.1, 0.5)$ shows that EDE-DeepONet can predict the solution well out of the training range. We compared the solution with 4 different initial condition parameter $a$ every 100 steps until the final time $T=0.04$ as Figure 7. Each row presents the solution under the same initial condition but with different evolution time $T$. With this example, it shows that EDE-DeepONet can deal with the PDE with a jump, while it is hard for other neural networks.
\begin{table}[tbh]
\centering
    $\begin{array}{||c|ccccc||}
\hline
\text { Error } & & T = 0.01 & T = 0.02 & T = 0.03 & T = 0.04 \\
\hline \hline a = 0.1  &  & 4.95 \times 10^{-5} & 2.27 \times 10^{-4} & 4.97 \times 10^{-4} & 7.40 \times 10^{-4} \\
\hline a = 0.3 &  & 2.32 \times 10^{-4} & 4.62 \times 10^{-4} & 6.66 \times 10^{-4} & 7.58 \times 10^{-4} \\
\hline a = 0.2  &  & 1.25 \times 10^{-4} & 3.46 \times 10^{-4} & 4.86 \times 10^{-3} & 7.15 \times 10^{-4} \\
\hline a = 0.4 &  & 2.80 \times 10^{-4} & 4.83 \times 10^{-4} & 6.35 \times 10^{-4} & 7.09 \times 10^{-4} \\
\hline a = 0.6 &  & 6.45 \times 10^{-4} & 5.02 \times 10^{-4} & 4.90 \times 10^{-4} & 6.15 \times 10^{-4} \\
\hline
\end{array}$
\caption{1D Allen-Cahn equation: The initial condition of the Allen-Cahn equation is $f(x) = a\sin{(\pi x)}$. The error is defined by $\frac{1}{N_x}\sum_{k=1}^{N_x}({u(x_k) - \hat{u}(x_k)})^2$, where $N_x = 51$, $u$ is the solution obtained by EDE-DeepONet and $\hat{u}$ is the reference solution.}
\end{table}
\begin{figure}[tbh]
    \centering
    \includegraphics[width=0.8\textwidth]{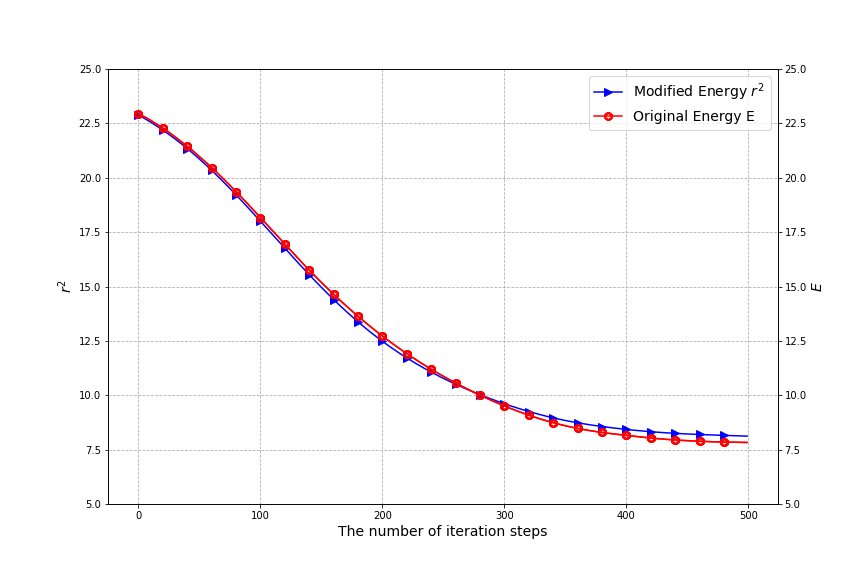}
     \label{fig:m=2, dt = 0.1}
    \caption{1D Allen-Cahn equation: The modified energy and original energy when training the network are shown above. Each iteration step represents one forward step of the PDE's numerical solution with $\Delta t = 10^{-4}$. The modified energy shows the same trends as the original energy.}
\end{figure}
\begin{figure}[tbh]
     \centering
     \begin{subfigure}{0.17\textwidth}
         \centering
         \includegraphics[width=\textwidth]{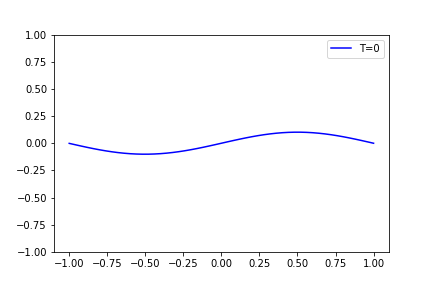}
         \caption{$a=0.1$, $T=0$}
         \label{fig:m=2, dt = 0.1}
     \end{subfigure}
     \begin{subfigure}{0.17\textwidth}
         \centering
         \includegraphics[width=\textwidth]{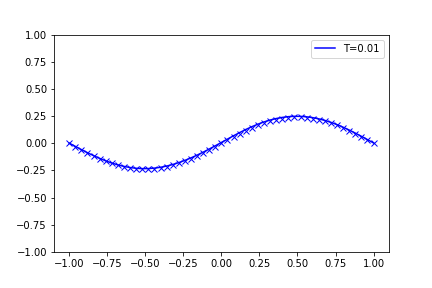}
         \caption{$a=0.1$, $T=0.01$}
         \label{fig:m=2, dt = 0.05}
     \end{subfigure}
     \begin{subfigure}{0.17\textwidth}
    \centering
    \includegraphics[width=\textwidth]{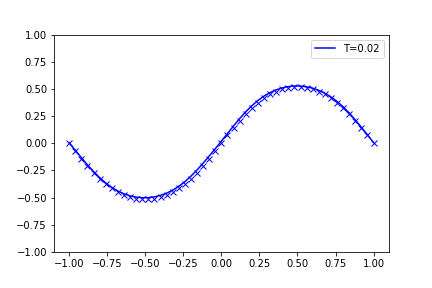}
    \caption{$a=0.1$, $T=0.02$}
     \label{fig:m=2, dt = 0.1}
    \end{subfigure}
    \begin{subfigure}{0.17\textwidth}
    \centering
    \includegraphics[width=\textwidth]{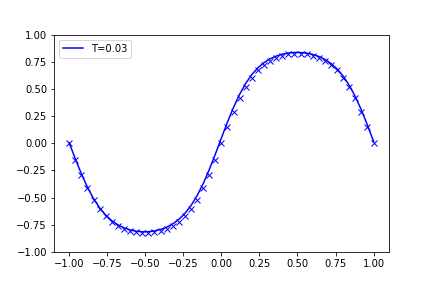}
    \caption{$a=0.1$, $T=0.03$}
     \label{fig:m=2, dt = 0.1}
    \end{subfigure}
    \begin{subfigure}{0.17\textwidth}
    \centering
    \includegraphics[width=\textwidth]{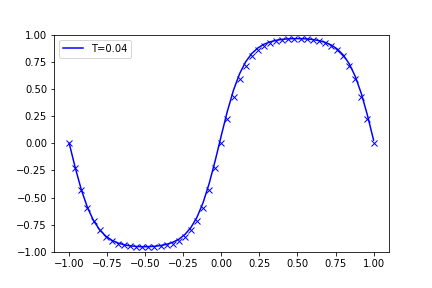}
    \caption{$a=0.1$, $T=0.04$}
     \label{fig:m=2, dt = 0.1}
    \end{subfigure}
    \begin{subfigure}{0.17\textwidth}
         \centering
         \includegraphics[width=\textwidth]{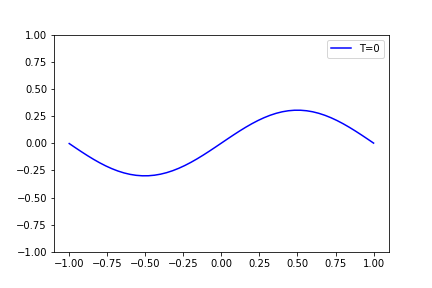}
         \caption{$a=0.3$, $T=0$}
         \label{fig:m=2, dt = 0.1}
     \end{subfigure}
     \begin{subfigure}{0.17\textwidth}
         \centering
         \includegraphics[width=\textwidth]{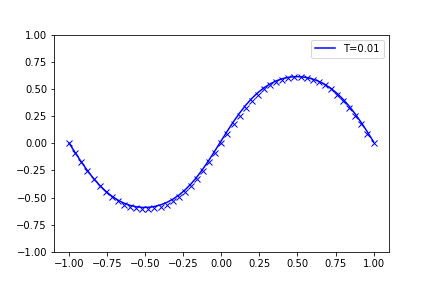}
         \caption{$a=0.3$, $T=0.01$}
         \label{fig:m=2, dt = 0.05}
     \end{subfigure}
     \begin{subfigure}{0.17\textwidth}
    \centering
    \includegraphics[width=\textwidth]{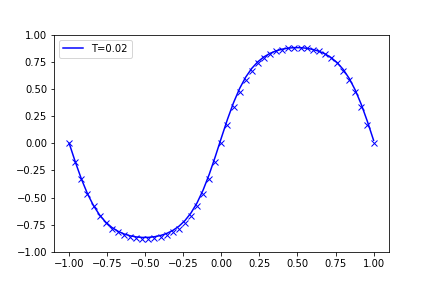}
    \caption{$a=0.3$, $T=0.02$}
     \label{fig:m=2, dt = 0.1}
    \end{subfigure}
    \begin{subfigure}{0.16\textwidth}
    \centering
    \includegraphics[width=\textwidth]{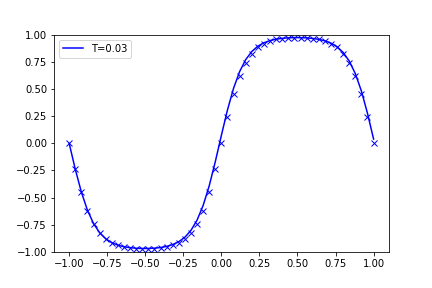}
    \caption{$a=0.3$, $T=0.03$}
     \label{fig:m=2, dt = 0.1}
    \end{subfigure}
    \begin{subfigure}{0.17\textwidth}
    \centering
    \includegraphics[width=\textwidth]{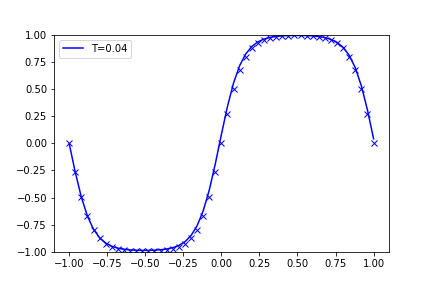}
    \caption{$a=0.3$, $T=0.04$}
     \label{fig:m=2, dt = 0.1}
    \end{subfigure}
    \begin{subfigure}{0.17\textwidth}
         \centering
         \includegraphics[width=\textwidth]{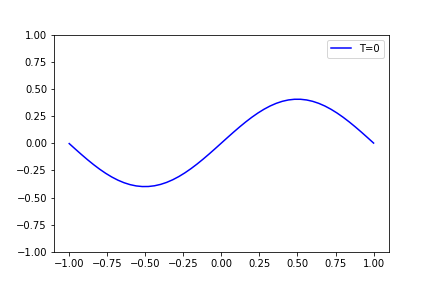}
         \caption{$a=0.4$, $T=0$}
         \label{fig:m=2, dt = 0.1}
     \end{subfigure}
     \begin{subfigure}{0.17\textwidth}
         \centering
         \includegraphics[width=\textwidth]{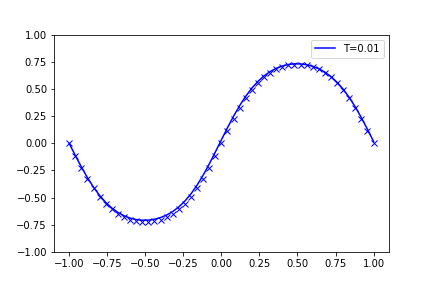}
         \caption{$a=0.4$, $T=0.01$}
         \label{fig:m=2, dt = 0.05}
     \end{subfigure}
     \begin{subfigure}{0.17\textwidth}
    \centering
    \includegraphics[width=\textwidth]{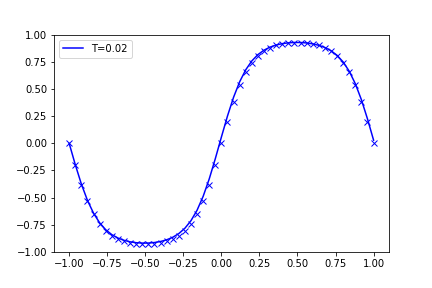}
    \caption{$a=0.4$, $T=0.02$}
     \label{fig:m=2, dt = 0.1}
    \end{subfigure}
    \begin{subfigure}{0.17\textwidth}
    \centering
    \includegraphics[width=\textwidth]{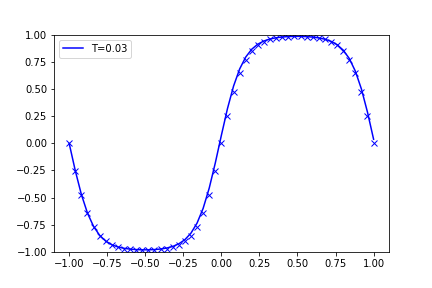}
    \caption{$a=0.4$, $T=0.03$}
     \label{fig:m=2, dt = 0.1}
    \end{subfigure}
    \begin{subfigure}{0.17\textwidth}
    \centering
    \includegraphics[width=\textwidth]{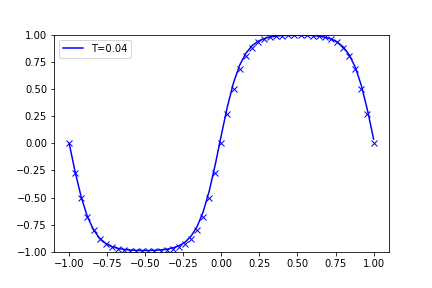}
    \caption{$a=0.4$, $T=0.04$}
     \label{fig:m=2, dt = 0.1}
    \end{subfigure}
    \begin{subfigure}{0.17\textwidth}
         \centering
         \includegraphics[width=\textwidth]{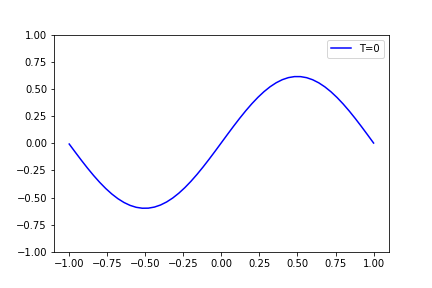}
         \caption{$a=0.6$, $T=0$}
         \label{fig:m=2, dt = 0.1}
     \end{subfigure}
     \begin{subfigure}{0.17\textwidth}
         \centering
         \includegraphics[width=\textwidth]{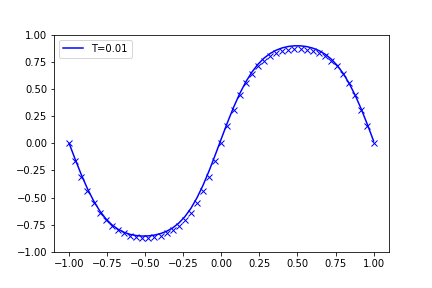}
         \caption{$a=0.6$, $T=0.01$}
         \label{fig:m=2, dt = 0.05}
     \end{subfigure}
     \begin{subfigure}{0.17\textwidth}
    \centering
    \includegraphics[width=\textwidth]{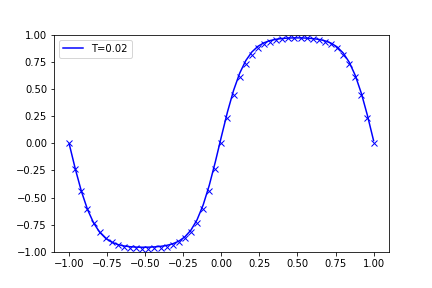}
    \caption{$a=0.6$, $T=0.02$}
     \label{fig:m=2, dt = 0.1}
    \end{subfigure}
    \begin{subfigure}{0.17\textwidth}
    \centering
    \includegraphics[width=\textwidth]{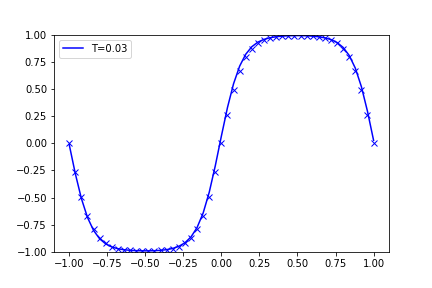}
    \caption{$a=0.6$, $T=0.03$}
     \label{fig:m=2, dt = 0.1}
    \end{subfigure}
    \begin{subfigure}{0.17\textwidth}
    \centering
    \includegraphics[width=\textwidth]{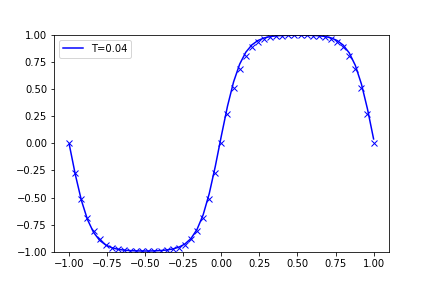}
    \caption{$a=0.6$, $T=0.04$}
     \label{fig:m=2, dt = 0.1}
    \end{subfigure}
    \caption{1d Allen-Cahn equation: The solution for 1d Allen-Cahn equation with 4 different initial conditions $f(x) = a\sin{\pi x}$. The curve represents the solution obtained by our model, and xxx represents the reference solution. We draw the figure for every 100 steps. The range of $a$ is $[0.1, 0.5]$. We also compare the solution with $a \notin [0.1, 0.5]$. All the figures show the trends of the phase separation.}
\end{figure}\\
(b) Various thickness of the interface:\\
Heuristically, $\epsilon$ represents the thickness of the interface in the phase separation process. We are able to obtain a sharp interface when $\epsilon \rightarrow 0$ with evolving in time. Each theoretical and numerical analysis of the limit makes a difference in the purpose of the understanding of the equation, cf. e.g. \cite{caginalp1998convergence, chen2006rapidly}. We  take $\epsilon$ as a training parameter. The problem can be described as:
\begin{align}
    &u_t =  u_{xx} - \frac{1}{\epsilon^2}(u^3 - u)&\\
    &u(-1,t) = u(1,t) = 0&
\end{align}
Since the training sample contains the parameter $\epsilon$, we can not use the same initial condition as the last example. We use spectral methods for a few steps with initial condition $u(x,0)=0.4\sin{(\pi x)}$. The training sample is generated based on the numerical solution of $u_{\epsilon}(x,0.02)$. We randomly select 50 different $\epsilon$ from $[0.1,0.2]$. We set the learning rate as $\Delta t = 10^{-4}$ and apply the adaptive time stepping strategy. We obtain the predicted solution after 400 iterations with different $\epsilon$. The rest setting is the same as the last example. The solution with different $\epsilon$ is shown in Figure 8. As $\epsilon$ goes smaller, the interface is sharper. Besides, the range of the training parameter is $(0.1,0.2)$. We are also able to obtain the solution out of the above range. EDE-DeepONet can track the limit of $\epsilon$ in only one training process, EDE-DeepONet can track the limit of $\epsilon$ in only one training process, whereas other traditional numerical methods hardly make it.
\begin{figure}[tbh]
     \centering
     \includegraphics[width=0.8\textwidth]{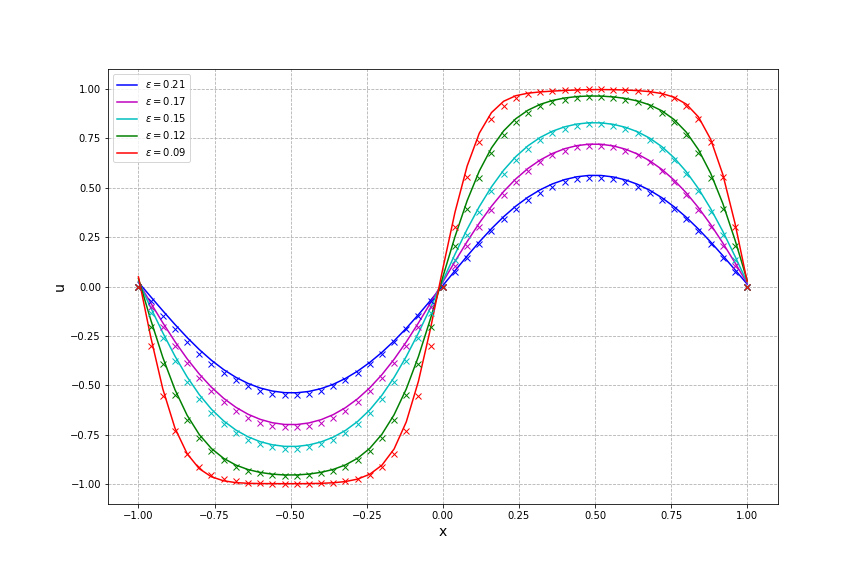}
     \label{fig:m=2, dt = 0.1}
    \caption{1d Allen-Cahn equation: Solutions with different thickness of the interface at the same final time. The curve represents the solution obtained by EDE-DeepONet. xxx represents the reference solution.}
\end{figure}
\subsubsection{2D case}
The 2D case Allen-Cahn equation is even more complex. The problem can be described as follows:
\begin{align}
    &u_t =  \Delta u - g(u)&\\
    &u(x,y,0) = a sin(\pi x)sin(\pi y)&\\
    &u(-1,y,t) = u(1,y,t) = u(x,-1,t) = u(x,1,t) = 0&
\end{align}
The corresponding  Ginzburg–Landau free energy $E[u] =  \int_{-1}^{1}\int_{-1}^{1} \frac{1}{2}( |u_x|^2 + |u_y|^2 ) dx dy + \int_{-1}^{1}\int_{-1}^{1}  G(u) dx$, where $ G(u) = \frac{1}{4\epsilon^2}(u^2 - 1)^2$ and $g(u) = G'(u) = \frac{1}{\epsilon^2}u(u^2 - 1)$. Usually, we take $\epsilon = 0.1$. In the training process, we take $\Delta t = 2\times 10^{-4}$. The number of spatial points is $51\times 51$ and the number of training parameters $a$ is 20. The way to choose $a \in (0.1, 0.4)$ and $x$ is the same as in example 1. We first compared the exact solution and the solution obtained by EDE-DeepONet with initial condition $f(x,y) = 0.2 sin(\pi x)sin(\pi y)$, where the exact solution is obtained by the traditional SAV method. EDE-DeepONet predicts the solution correctly based on Table 4 and Figure 9. Then in order to show its accuracy, we draw Figure 10 with more parameters. All the examples show good trends to separate. The case $a = 0.4$ is out of the training range, but it still approaches the exact solution.
\begin{table}[tbh]
\centering
    $\begin{array}{||c|cccc||}
\hline
\text { Error } & & T = 0.01 & T = 0.02 & T = 0.03 \\
\hline \hline a = 0.15 &  & 1.23 \times 10^{-4} & 6.53 \times 10^{-4} & 2.75 \times 10^{-3} \\
\hline a = 0.2  &  & 2.24 \times 10^{-4} & 1.10 \times 10^{-3} & 4.04 \times 10^{-3} \\
\hline a = 0.3 &  & 4.28 \times 10^{-4} & 1.84 \times 10^{-3} & 5.76 \times 10^{-3} \\
\hline a = 0.35  &  & 5.31 \times 10^{-4} & 2.17 \times 10^{-3} & 6.25 \times 10^{-3} \\
\hline a = 0.4  &  & 6.94 \times 10^{-4} & 2.71 \times 10^{-3} & 7.22 \times 10^{-3} \\
\hline
\end{array}$
\caption{2D Allen-Cahn equation: The initial condition of the 2D Allen-Cahn equation is $f(x, y) = a\sin{(\pi x)}\sin{(\pi y)}$. The error is defined by $\frac{1}{N_xN_y}\sum_{k=1}^{N_x}\sum_{j=1}^{N_y}({u(x_k,y_j) - \hat{u}(x_k,y_j)})^2,$ where $N_x = N_y = 51$, $u$ is the solution obtained by EDE-DeepONet and $\hat{u}$ is the reference solution.}
\end{table}
\begin{figure}
    \centering
    \begin{subfigure}{0.24\textwidth}
    \centering
    \includegraphics[width=\textwidth]{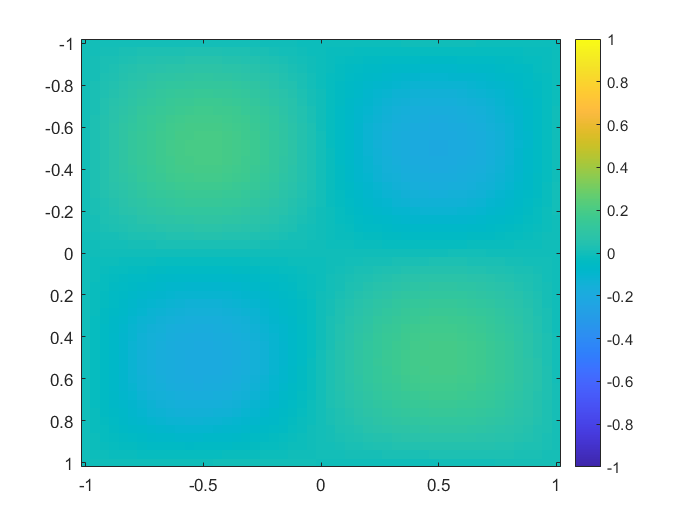}
    \caption{$a=0.2$, $T=0$}
     \label{fig:m=2, dt = 0.1}
    \end{subfigure}
     \begin{subfigure}{0.24\textwidth}
         \centering
         \includegraphics[width=\textwidth]{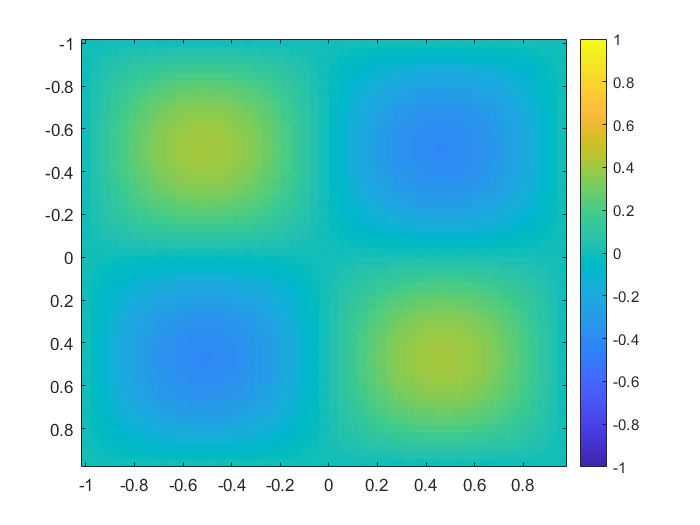}
         \caption{$a=0.2$, $T=0.01$}
         \label{fig:m=2, dt = 0.1}
     \end{subfigure}
     \begin{subfigure}{0.24\textwidth}
         \centering
         \includegraphics[width=\textwidth]{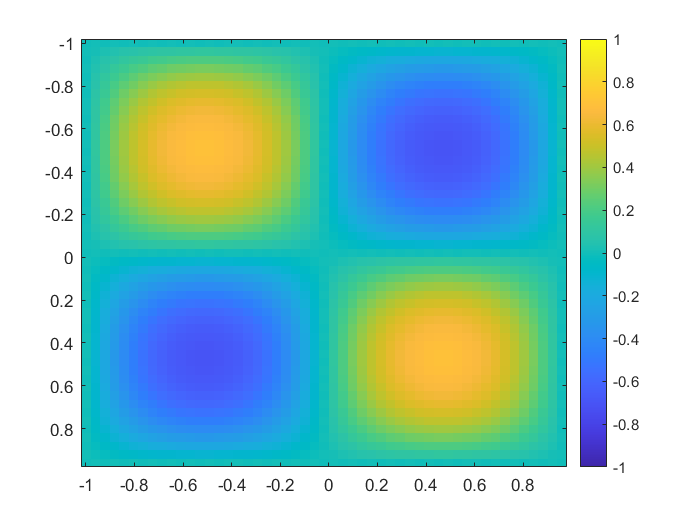}
         \caption{$a=0.2$, $T=0.02$}
         \label{fig:m=2, dt = 0.05}
     \end{subfigure}
     \begin{subfigure}{0.24\textwidth}
    \centering
    \includegraphics[width=\textwidth]{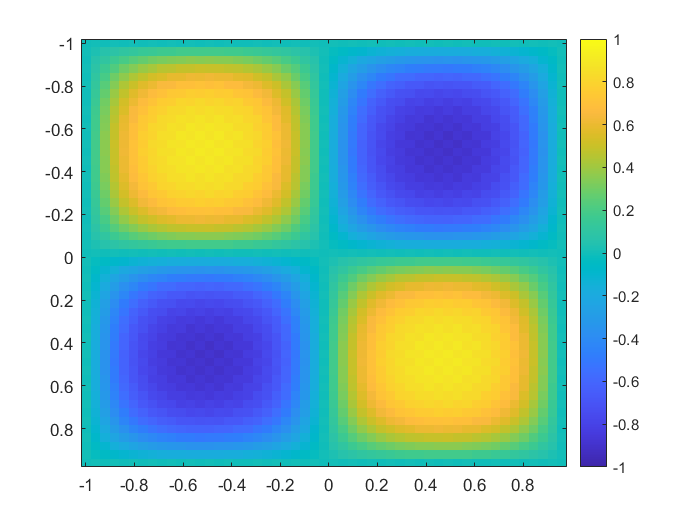}
    \caption{$a=0.2$, $T=0.03$}
     \label{fig:m=2, dt = 0.1}
    \end{subfigure}
        \begin{subfigure}{0.24\textwidth}
         \centering
         \includegraphics[width=\textwidth]{2allen21.png}
         \caption{$a=0.2$, $T=0$}
         \label{fig:m=2, dt = 0.1}
     \end{subfigure}
     \begin{subfigure}{0.24\textwidth}
         \centering
         \includegraphics[width=\textwidth]{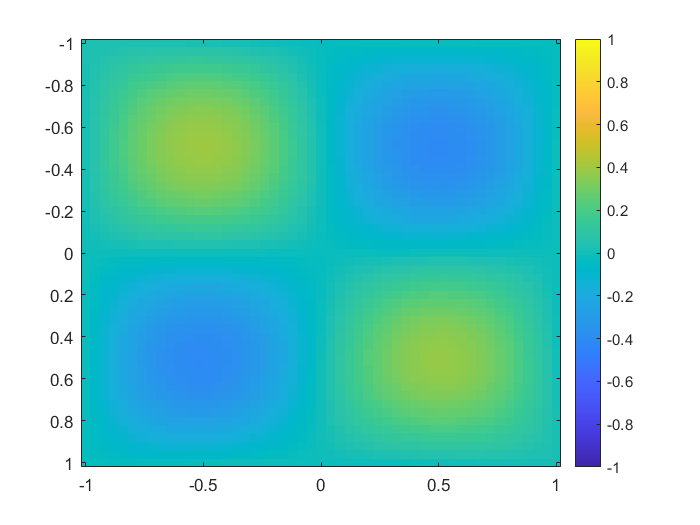}
         \caption{$a=0.2$, $T=0.01$}
         \label{fig:m=2, dt = 0.05}
     \end{subfigure}
     \begin{subfigure}{0.24\textwidth}
    \centering
    \includegraphics[width=\textwidth]{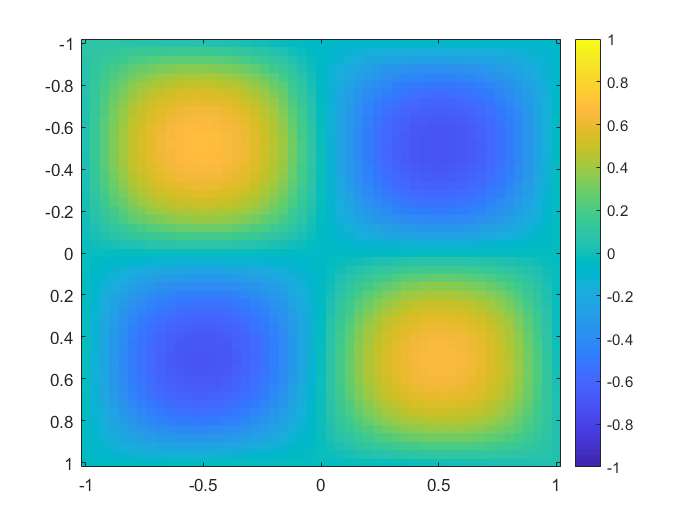}
    \caption{$a=0.2$, $T=0.02$}
     \label{fig:m=2, dt = 0.1}
    \end{subfigure}
    \begin{subfigure}{0.24\textwidth}
    \centering
    \includegraphics[width=\textwidth]{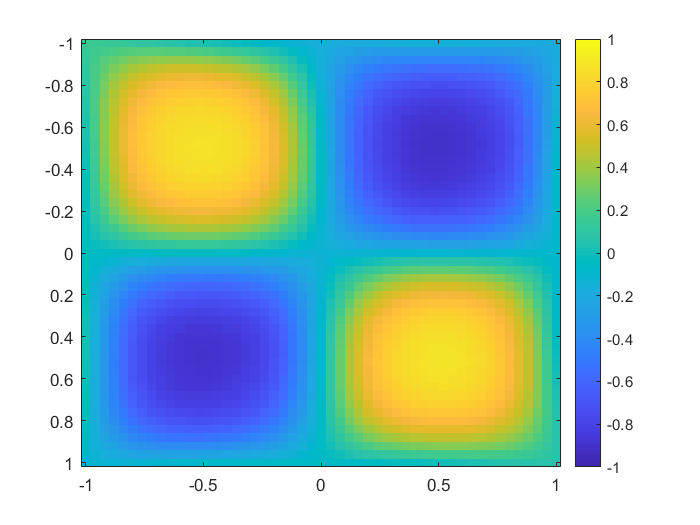}
    \caption{$a=0.2$, $T=0.03$}
     \label{fig:m=2, dt = 0.1}
    \end{subfigure}
    \caption{2D Allen-Cahn equation: (a)-(d) represents the reference solution of the 2D Allen-Cahn equation with initial condition $f(x,y) = 0.3\sin{(\pi x)}\sin{(\pi y)}$. (e)-(h) is the solution obtained by the EDE-DeepONet.}
\end{figure}
\begin{figure}[tbh]
     \centering
     \begin{subfigure}{0.24\textwidth}
         \centering
         \includegraphics[width=\textwidth]{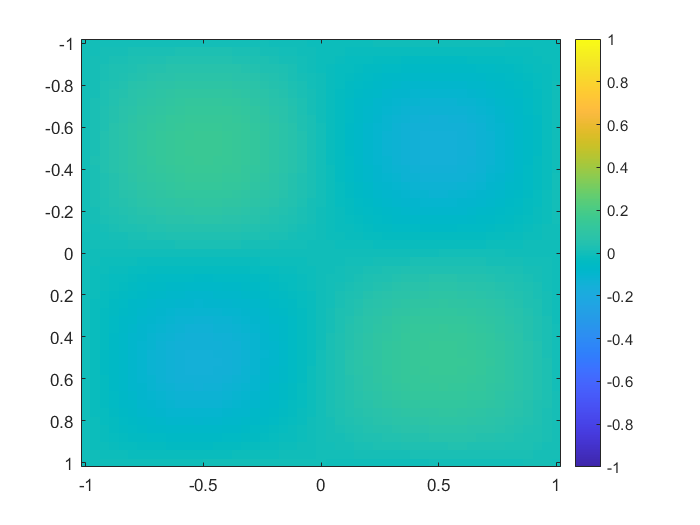}
         \caption{$a=0.15$, $T=0$}
         \label{fig:m=2, dt = 0.1}
     \end{subfigure}
     \begin{subfigure}{0.24\textwidth}
         \centering
         \includegraphics[width=\textwidth]{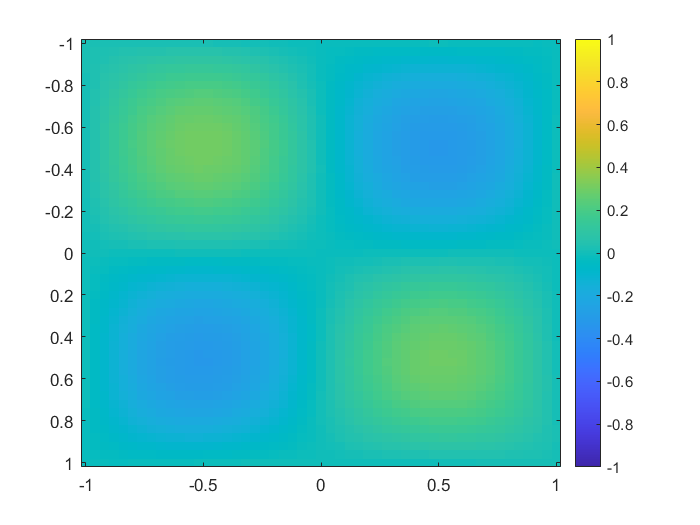}
         \caption{$a=0.15$, $T=0.01$}
         \label{fig:m=2, dt = 0.05}
     \end{subfigure}
     \begin{subfigure}{0.24\textwidth}
    \centering
    \includegraphics[width=\textwidth]{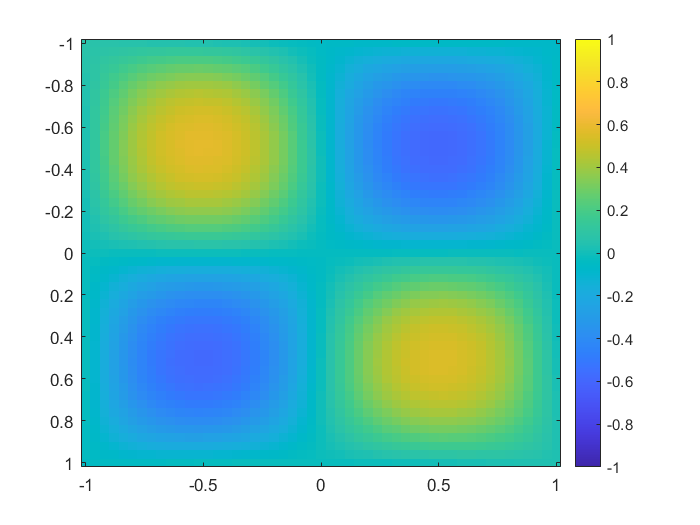}
    \caption{$a=0.15$, $T=0.02$}
     \label{fig:m=2, dt = 0.1}
    \end{subfigure}
    \begin{subfigure}{0.24\textwidth}
    \centering
    \includegraphics[width=\textwidth]{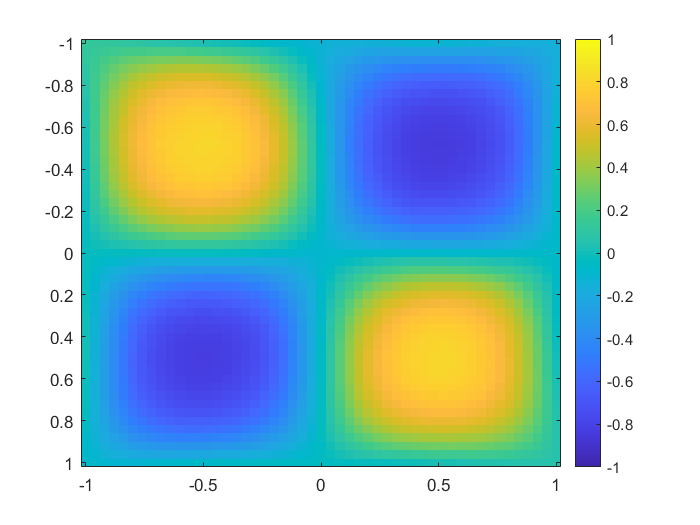}
    \caption{$a=0.15$, $T=0.03$}
     \label{fig:m=2, dt = 0.1}
    \end{subfigure}
    \begin{subfigure}{0.24\textwidth}
         \centering
         \includegraphics[width=\textwidth]{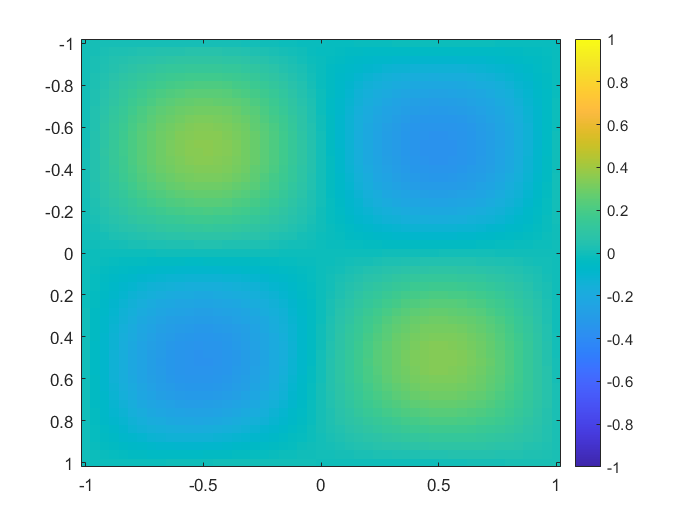}
         \caption{$a=0.35$, $T=0$}
         \label{fig:m=2, dt = 0.1}
     \end{subfigure}
     \begin{subfigure}{0.24\textwidth}
         \centering
         \includegraphics[width=\textwidth]{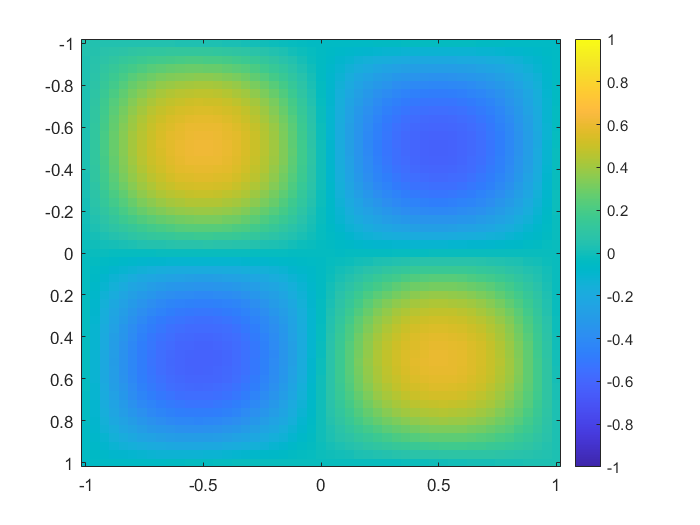}
         \caption{$a=0.35$, $T=0.01$}
         \label{fig:m=2, dt = 0.05}
     \end{subfigure}
     \begin{subfigure}{0.24\textwidth}
    \centering
    \includegraphics[width=\textwidth]{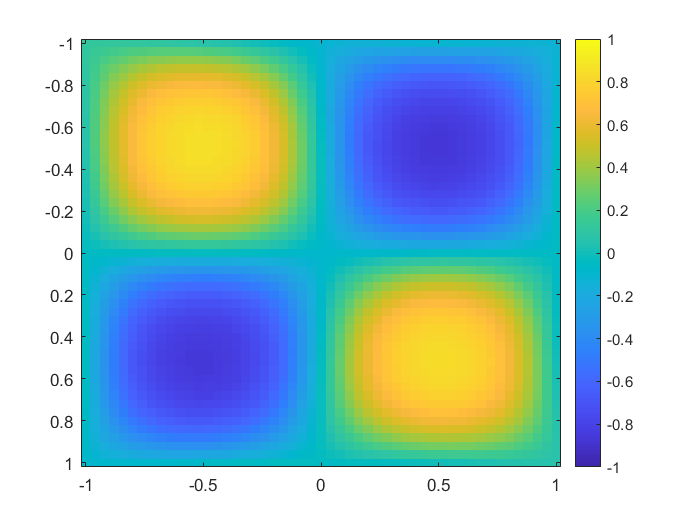}
    \caption{$a=0.35$, $T=0.02$}
     \label{fig:m=2, dt = 0.1}
    \end{subfigure}
    \begin{subfigure}{0.24\textwidth}
    \centering
    \includegraphics[width=\textwidth]{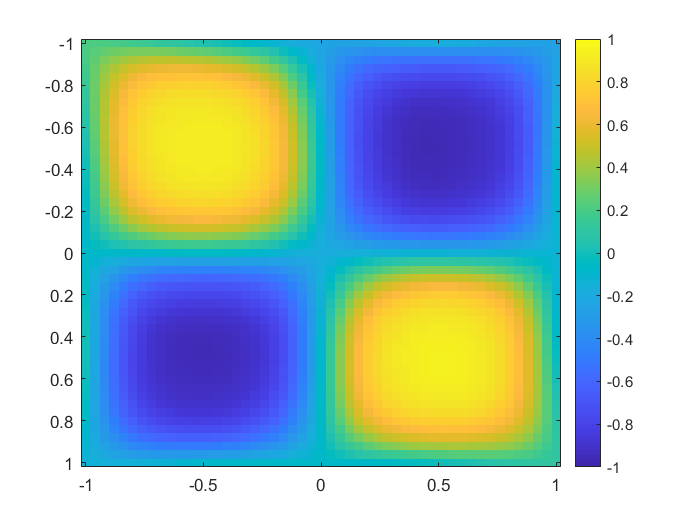}
    \caption{$a=0.35$, $T=0.03$}
     \label{fig:m=2, dt = 0.1}
    \end{subfigure}
    \caption{2D Allen-Cahn equation: The solution of 2D Allen-Cahn equation with 4 different initial conditions $f(x,y) = a \sin{\pi x}\sin{\pi y}$. The training parameter $a \in [0.1, 0.4]$. We draw three figures where $a$ is in the training range and one figure where $a$ is out of the training range. All the figures show the phase separation trends according to the reference solution. As $a$ is further away from the training range, the error tends to be larger.}
\end{figure}
\section{Concluding Remarks}
In this paper, we provide a new neural network architecture to solve parametric PDEs with different initial conditions, while maintaining the energy dissipative of dynamic systems. We first introduce the energy dissipative law of dynamic systems to the DeepONet. We also introduce an adaptive time stepping strategy and restart strategy. With our experiments, both above strategies help keep the modified energy approaching the original energy. To avoid much cost of training the DeepONet, we evolve the neural network based on Euler methods. In this article, we adopt the SAV method to solve gradient flow problems. With this successful attempt, more work could be done. For example, we can consider a general Wasserstein gradient flow problem. We are only adopting the basic architecture of the DeepONet. The more advanced architecture is compatible to our work. It may further improve the accuracy of EDE-DeepONet. 

\section*{Acknowledgments}
SJ and SZ gratefully acknowledge the support of NSF DMS-1720442 and AFOSR FA9550-20-1-0309. GL and ZZ gratefully acknowledge the support of the National Science Foundation (DMS-1555072, DMS-2053746, and DMS-2134209), Brookhaven National Laboratory Subcontract 382247, and U.S. Department of Energy (DOE) Office of Science Advanced Scientific Computing Research program DE-SC0021142 and DE-SC0023161.
\clearpage
\bibliographystyle{model1-num-names}
\bibliography{refs}

\end{document}